\documentclass{article}
\PassOptionsToPackage{round}{natbib}

\usepackage[preprint]{neurips_2026}

\usepackage[american]{babel}
\usepackage{microtype}
\usepackage{graphicx}
\usepackage{subfigure}
\usepackage{multirow}
\usepackage{bbm}
\usepackage{hyperref}
\usepackage{wrapfig}

\usepackage{amsmath}
\usepackage{array}
\usepackage{booktabs}
\usepackage{longtable}

\usepackage[dvipsnames]{xcolor}
\usepackage{mathtools} 
\usepackage{booktabs} 

\usepackage{algorithm}
\usepackage[noEnd=true, indLines=true, endLComment=~]{algpseudocodex}
\algrenewcommand\algorithmicrequire{\textbf{Inputs:}}
\algrenewcommand\algorithmicensure{\textbf{Outputs:}}

\usepackage{amsmath,amsfonts,bm}

\def\eqref#1{equation~\ref{#1}}

\def\1{\bm{1}}

\DeclareMathAlphabet{\mathsfit}{\encodingdefault}{\sfdefault}{m}{sl}
\SetMathAlphabet{\mathsfit}{bold}{\encodingdefault}{\sfdefault}{bx}{n}

\newcommand{\E}{\mathbb{E}}

\newcommand{\R}{\mathbb{R}}

\newcommand{\sigmoid}{\sigma}

\newcommand{\KL}{D_{\mathrm{KL}}}
\newcommand{\Var}{\mathrm{Var}}

\usepackage{url}
\usepackage{amsmath}
\usepackage{amssymb}
\usepackage{mathrsfs}
\usepackage{amsthm}
\usepackage{lipsum}
\usepackage{thmtools}
\usepackage{thm-restate}
\usepackage{adjustbox}
\usepackage{caption}

\usepackage{tabularx}

\usepackage[mode=buildnew]{standalone}
\usepackage{tikz}
\usepackage{pgfplots}

\usepackage[capitalize,noabbrev]{cleveref}

\usepgfplotslibrary{groupplots}
\pgfplotsset{compat=1.18}   
\definecolor{truedir}{RGB}{0,158,115}
\definecolor{revdir}{RGB}{213,94,0}
\definecolor{noedge}{RGB}{153,153,153}

\DeclareMathOperator{\pa}{Pa}        
      
\newcommand{\Id}{I_m}
\newcommand{\Ones}{\mathbf{1}_{m\times m}}

\definecolor{mydarkblue}{rgb}{0,0.08,0.45}
\hypersetup{ %
    pdftitle={},
    pdfauthor={},
    pdfsubject={},
    pdfkeywords={},
    pdfborder=0 0 0,
    pdfpagemode=UseNone,
    colorlinks=true,
    linkcolor=mydarkblue,
    citecolor=mydarkblue,
    filecolor=mydarkblue,
    urlcolor=mydarkblue,
}

\theoremstyle{plain}

\theoremstyle{definition}

\theoremstyle{remark}

\hypersetup{colorlinks=true,citecolor=blue,urlcolor=blue}

\crefformat{section}{#2Section~#1#3}
\crefformat{appendix}{#2Appendix~#1#3}

\crefformat{equation}{(#2#1#3)}
\crefmultiformat{equation}{#2Equations~#1#3}%
{ \&~#2#1#3}{, #2#1#3}{, \&~#2#1#3}

\crefformat{table}{#2Table~#1#3}
\crefformat{figure}{#2Figure~#1#3}

\crefformat{theorem}{#2Theorem~#1#3}
\crefmultiformat{theorem}{#2Theorems~#1#3}%
{ \&~#2#1#3}{, #2#1#3}{, and~(#2#1#3)}

\crefformat{lemma}{#2Lemma~#1#3}
\crefformat{proposition}{#2Proposition~#1#3}
\crefmultiformat{proposition}{#2Propositions~#1#3}%
{ \&~#2#1#3}{, #2#1#3}{, and~(#2#1#3)}

\crefformat{algorithm}{#2Algorithm~#1#3}
\crefmultiformat{algorithm}{#2Algorithms~#1#3}%
{ \&~#2#1#3}{, #2#1#3}{, and~(#2#1#3)}

\crefformat{corollary}{#2Corollary~#1#3}
\crefformat{definition}{#2Definition~#1#3}

\crefformat{assumption}{#2Assumption~#1#3}
\crefmultiformat{assumption}{#2Assumptions~#1#3}%
{ \&~#2#1#3}{, #2#1#3}{, and~(#2#1#3)}

\usepackage{pifont} 

\title{\texttt{SVI-DAG}: A Structured Variational Inference Approach to Bayesian Causal Discovery}

\author{%
\textbf{Shrenik Zinage\thanks{This work was conducted while the author was a PhD student at Purdue University.}} \\
Massachusetts Institute of Technology \\
\texttt{shrenik@mit.edu}
}

\begin{document}

\maketitle

\begin{abstract}
Bayesian causal discovery seeks to determine the posterior distribution of causal theories, which are interpreted as directed acyclic graphs (DAGs) that explain the observed data. The resulting posterior allows systematic reasoning regarding epistemic uncertainty within these theories. Nonetheless, finding such graphs is difficult due to identifiability problems and limited observational data. Furthermore, precisely approximating posterior over graphs is challenging given vast range of potential DAGs. Recent Bayesian approaches have addressed some of these challenges, yet they remain limited as they fail to encode dependencies between edges, and lack principled ways to incorporate domain knowledge as inductive biases during the search process. To overcome these limitations, we propose \texttt{SVI-DAG}, a structured variational inference approach to Bayesian causal discovery using observational data and prior beliefs that uses normalizing flows to model dependencies between edges, supporting expressive and multimodal posterior learning over DAGs. 
To mitigate mode seeking behaviour in evidence lower bound optimization and promote mode coverage, we use stein variational gradient descent to update the node potentials using a kernel in acyclicity space. 
We evaluate \texttt{SVI-DAG} against 5 state-of-the-art Bayesian DAG learning methods and demonstrate competitive performance in terms of both accuracy and uncertainty quantification.
The code is available at \url{https://github.com/shrenikvz/SVI-DAG}.
\end{abstract}

\section{Introduction}
\label{sec:introduction}
Understanding causal relationships from data~\citep{scholkopf2021toward} remains a central problem across science and engineering. Causal discovery---the task of inferring which variables directly influence others---allows researchers to move beyond mere associations and reason about the consequences of interventions in complex systems. Such a capability holds the potential to accelerate progress across a broad range of domains, including epidemiology~\citep{tennant2021use}, climate science~\citep{runge2019inferring}, economics~\citep{imbens2020potential}, psychology~\citep{foster2010causal}, and beyond.
Causal relationships are commonly represented as directed acyclic graphs (DAGs)~\citep{peters2017elements}, where nodes correspond to variables and directed edges encode direct causal influences.
Learning the DAG from data, however, constitutes a formidable combinatorial problem~\citep{heinze2018causal}, as the number of admissible DAGs grows super-exponentially with the number of variables and with purely observational data, the true DAG may be identifiable only up to its Markov equivalence class (MEC)\footnote{The Markov equivalence class is the set of DAGs that encode the same conditional independencies.}. This fundamental identifiability limitation motivates a Bayesian treatment~\citep{friedman2003being}, wherein the goal is to approximate the posterior distribution over all DAGs that explain the data.
Recent work on safe and trustworthy AI~\citep{bengio2026international} further argues that epistemic uncertainty over causal theories is essential for building non-agentic Scientist AI~\citep{bengio2025superintelligent} systems that provide reliable predictions and can serve as guardrails against overconfident decision making.

Existing Bayesian approaches to DAG learning have made substantial progress, yet several limitations remain. A predominant assumption is that edges are modeled as independent random variables, thereby disregarding the rich dependencies that exist between edges. Failing to capture these dependencies leads to difficulty in capturing the multimodal structure of posterior. Moreover, existing Bayesian methods largely lack principled mechanisms for incorporating domain knowledge, which limits their capacity to guide the search toward structurally plausible regions of hypothesis space. 

\textbf{Our Approach.} To address these shortcomings, we introduce \texttt{SVI-DAG}, a structured variational inference (SVI) approach to Bayesian causal discovery. Our \textbf{core contributions} are as follows:

\begin{itemize}
    \item We propose a differentiable Bayesian approach to causal discovery that encodes dependencies between edge logits using conditional normalizing flows to support expressive and multimodal learning over DAGs. 
    \item Our approach provides a principled way to incorporate domain knowledge as an inductive bias using a new type of prior that adapts based on how strong or weak the prior beliefs are using a Beta-Bernoulli distribution. We further show that this construction induces a tractable Logistic-Beta prior over the edge logits.
    \item To mitigate mode seeking behavior in evidence lower bound (ELBO) optimization, we use stein variational gradient descent (SVGD) to update node potentials using kernel in acyclicity space. This improves mode coverage and quantification of epistemic uncertainty.
\end{itemize}

\section{Related Works}
\label{sec:related_works}
Given the vast literature on causal discovery~\citep{squires2023causal}, we focus exclusively on approaches that pose causal discovery as an optimization problem~\citep{vowels2022d}.

\paragraph{Causal Discovery.}
A foundational shift in causal discovery began with NOTEARS~\citep{zheng2018dags}, which reformulated combinatorial DAG search as a continuous optimization problem. This idea was extended to nonlinear settings using graph neural networks~\citep{yu2019dag}, gradient based neural DAG learning~\citep{lachapelle2019gradient}, and graph autoencoder architectures~\citep{ng2019graph}. Further improvements include DAGs with No Curl~\citep{yu2021dags} which improves efficiency via a curl free condition, 
and methods using acyclicity such as DAGMA~\citep{bello2022dagma}. Complementary approaches introduced masked gradients~\citep{ng2022masked} or analyzed interaction of sparsity and acyclicity penalties~\citep{ng2020role} while NoDAGs-Flow~\citep{sethuraman2023nodags} relaxed the acyclicity assumption entirely, learning nonlinear cyclic causal structures via normalizing flows. 
On combinatorial side, BOSS~\citep{andrews2023fast} demonstrated that efficient order score search can rival continuous relaxation methods in speed and accuracy, while \citet{ban2024differentiable} proposed differentiable structure learning over partial topological orderings that flexibly interpolates between unconstrained and fully ordered search. 
Bridging constraint based and score based paradigms, \citet{zhou2025differentiable} reformulated conditional independence testing differentiably, allowing end-to-end gradient based learning with statistical guarantees of constraint based methods. 
These methods demonstrate strong empirical performance but typically return point estimates and do not quantify uncertainty over graph structures.

\paragraph{Bayesian Causal Discovery.}
For Bayesian approaches, early differentiable methods include D-VAE~\citep{zhang2019d}, which encodes DAGs in a latent space via variational autoencoders. VCN~\citep{annadani2021variational} and BCD Nets~\citep{cundy2021bcd} introduced variational inference (VI) formulations over graph structures, while DiBS~\citep{lorch2021dibs} proposed a differentiable particle based framework. Concurrent lines of work explored differentiable DAG sampling~\citep{charpentier2022differentiabledag}, tractable marginal uncertainty computation~\citep{wang2022tractable}, and amortized inference for both static~\citep{lorch2022amortized} and temporal~\citep{lowe2022amortized} settings. GFlowNets emerged as a powerful alternative, allowing diverse posterior sampling over structures~\citep{deleu2022bayesian} and joint inference over graphs and parameters~\citep{nishikawa2022bayesian, deleu2023joint}. More recently, BayesDAG~\citep{annadani2023bayesdag} introduced gradient based posterior inference with improved scalability, ProDAG~\citep{thompson2024prodag} proposed projected VI to enforce acyclicity, and meta learning formulations~\citep{dhir2024meta} have been explored to generalize across tasks.

\section{Methodology}
\label{sec:methodology}

\subsection{Preliminaries and Notations}
\label{subsec:preliminaries}

\subsubsection{Causal Graph and Structural Causal Model}

Let \(\mathcal{G}=(V,E)\) be a DAG on the variable set \(V=\{X_1,\dots,X_m\}\) where $X_i$ represents a random variable.
Let \(A\in\{0,1\}^{m\times m}\) be the adjacency matrix encoding the causal relationships among the \(m\) variables with entries  $A_{ij} = 1$ if there is a directed edge $X_j \to X_i \text{ in } \mathcal{G}$ and 0 otherwise.
Let \(\mathcal{X}_i \subseteq \mathbb{R}\) denote the sample space for node \(X_i\), and define the joint sample space
\(\mathcal{X}= \mathcal{X}_1\times\cdots\times \mathcal{X}_m\).
Let \(\mathcal{D}=\{x^{(n)}\}_{n=1}^N\) be the dataset, where each observation is a tuple
\(x^{(n)}=(x^{(n)}_1,\dots,x^{(n)}_m)\in \mathcal{X}\) with \(x^{(n)}_i\in \mathcal{X}_i\). A structural causal model (SCM) states that each variable is generated as a function of its parents and an exogenous noise.
We denote by $X_{{\pa}(i)}$ the set of parental nodes of $X_i$ so that there is an edge from $X_j \in X_{{\pa}(i)}$ to $X_i$ in DAG $\mathcal{G}$.
Formally, for each node $i\in\{1,\dots,m\}$, 
$$
X_i = f_i\bigl(X_{{\pa}(i)}; \theta_i\bigr) + \epsilon_i,
$$
where $f_i(\cdot;\theta_i)$ is a deterministic function with parameters $\theta_i$ and
and  $\epsilon_i$'s are independent noise variables with strictly positive densities w.r.t Lebesgue measure. 
If noise variables are Gaussian and functions $f_i$ are not linear or constant,
then additive noise model (ANM) is structurally identifiable \citep{peters2014causal,hoyer2008nonlinear}. Causality typically assumes structural
assignments do not form cycles and they induce a DAG \citep{pearl2009causality}.

\subsubsection{Bayesian Structure Learning}

Our goal is to approximate the posterior distribution over all possible DAGs \(A\) and associated parameters
\(\theta = \{\theta_1,\dots,\theta_{m}\}\) in a neural network based SCM (NCM). Specifically, we target the posterior over structures and parameters
$
p(A,\theta\mid \mathcal{D})\ \propto\ p(\mathcal{D}\mid A,\theta)\,p(\theta\mid A)\,p(A)
$
with a prior over parameters $p(\theta\mid A)$ and graphs $p(A)$ \citep{friedman2003being}.
The likelihood is:
\[
p(\mathcal{D}\mid A,\theta)
=
\prod_{n=1}^{N}\ \prod_{i=1}^m p\bigl(x_i^{(n)}\mid x_{\pa(i)}^{(n)};\theta_i\bigr),
\]
where the dependency on $A$ is implicit in the parent index sets $\pa(i)$.
Concretely, assuming the noise is Gaussian,
$
p(X_i \mid X_{\pa(i)}, \theta_i) \;=\; \mathcal{N}\bigl(X_i \mid f_i(X_{\pa(i)}; \theta_i), \sigma_i^2\bigr).
$
Each \(\theta_i\) determines conditional distribution of \(X_i\) given its parents respecting the DAG structure \(A\). 

\subsection{Our Approach}

In this section, we detail each component of the $\texttt{SVI-DAG}$ algorithm which is specifically designed to address some of the critical challenges in causal structure learning.

\subsubsection{Domain Informed Prior over Adjacency Matrix}

We start by defining a prior over the DAGs. A common choice is 
$
B_{ij}\sim \mathrm{Bernoulli}(p_{ij}).
$
The parameters \(p_{ij}\) encode prior beliefs about edge presence. To improve robustness to misspecified prior beliefs, we introduce a hierarchical prior over \(B\). For each ordered pair \((i,j)\) with \(i\neq j\), we introduce \(\pi_{ij}\in(0,1)\) such that
$p(B_{ij}\mid \pi_{ij}) = \pi_{ij}^{B_{ij}} (1-\pi_{ij})^{1-B_{ij}}.$
Instead of fixing \(\pi_{ij}\) to the prior value \(p_{ij}\),
we place a hyper prior on \(\pi_{ij}\) using a Beta distribution $\pi_{ij}\sim \mathrm{Beta}\Bigl(\alpha_{ij},\beta_{ij}\Bigr)$. We choose hyperparameters so that Beta-density mode is \(p_{ij}\) given by
\begin{equation*}
\alpha_{ij} = \nu_{ij} p_{ij} + 1,\quad 
\beta_{ij} = \nu_{ij}(1 - p_{ij}) + 1,\quad 
\nu_{ij} = \nu_{\min} + \kappa \,\lvert p_{ij} - 0.5\rvert^{\eta}\ \text{for } i \neq j,
\end{equation*}
where \(\nu_{ij} > 0\) is a concentration parameter that is strictly increasing in \(\lvert p_{ij}-0.5\rvert\) when \(\kappa>0\) and \(\nu_{\min}\), \(\kappa\), and \(\eta\) are constants. 
This choice assigns higher concentration to nominal beliefs farther from 0.5 (see Figure \ref{fig:adaptive_domain_prior}). 
As \(\nu_{ij}\to\infty\), the marginal Beta--Bernoulli prior over each free edge $B_{ij}$ converges to a \(\mathrm{Bernoulli}(p_{ij})\) prior (see Proposition~\ref{prop:beta-bernoulli-concentration} in appendix \ref{app:supporting_statements_and_proofs}).
Let \(\Pi = \{\pi_{ij}\}_{i\neq j}\). We 
sample edges as \(B_{ij}\sim \mathrm{Bernoulli}(\pi_{ij})\), which gives the joint prior as
\begin{equation*}
p(B,\Pi) \propto \prod_{i\neq j} \mathrm{Bernoulli}(B_{ij}\mid\pi_{ij})\, \mathrm{Beta}\Bigl(\pi_{ij}\mid\nu_{ij}\,p_{ij}+1,\;\nu_{ij}\,(1-p_{ij})+1\Bigr).
\end{equation*}
\begin{wrapfigure}[17]{r}{0.4\textwidth}
    \vspace{0em}
    \centering
    \scalebox{0.4}{
        \begin{tikzpicture}[scale=1.2, >=stealth]

\pgfplotsset{compat=1.18}

\definecolor{C0}{HTML}{1f77b4}
\definecolor{C1}{HTML}{ff7f0e}
\definecolor{C2}{HTML}{2ca02c}
\definecolor{C3}{HTML}{d62728}
\definecolor{C4}{HTML}{9467bd}
\definecolor{C5}{HTML}{8c564b}
\definecolor{C6}{HTML}{e377c2}

\begin{axis}[
    width=10.5cm,
    height=7.8cm,
    xmin=0, xmax=1,
    ymin=0,
    domain=0.001:0.999,
    samples=400,
    smooth,
    grid=major,
    grid style={dashed, opacity=0.6},
    axis line style={line width=1.2pt},
    tick style={line width=1pt},
    ticklabel style={font=\small},
    xlabel={$p$},
    ylabel={Density},
    xlabel style={font=\Large},
    ylabel style={font=\Large},
    title={$\nu = \nu_{\min} + \kappa |p - 0.5|^{\eta}, (\nu_{min} = 1, \kappa = 500, \eta = 2)$},
    title style={font=\Large, yshift=1ex},
    legend columns=2,
    legend style={
        at={(0.5,-0.22)},
        anchor=north,
        draw=black,
        font=\normalsize,
        row sep=2pt,
        column sep=6pt
    },
    cycle list={
        {C0, line width=1.8pt},
        {C1, line width=1.8pt},
        {C2, line width=1.8pt},
        {C3, line width=1.8pt},
        {C4, line width=1.8pt},
        {C5, line width=1.8pt},
        {C6, line width=1.8pt}
    },
    declare function={
        lgamma(\x) = (\x-0.5)*ln(\x) - \x + 0.5*ln(2*pi) + 1/(12*\x);
        betapdf(\x,\a,\b) = exp((\a-1)*ln(\x) + (\b-1)*ln(1-\x) - lgamma(\a) - lgamma(\b) + lgamma(\a+\b));
    }
]

\addplot {betapdf(x, 6.1125, 98.1375)};
\addlegendentry{$p=0.05,\ \nu=102.2$}

\addplot {betapdf(x, 9.1, 73.9)};
\addlegendentry{$p=0.1,\ \nu=81.0$}

\addplot {betapdf(x, 7.3, 15.7)};
\addlegendentry{$p=0.3,\ \nu=21.0$}

\addplot {betapdf(x, 1.5, 1.5)};
\addlegendentry{$p=0.5,\ \nu=1.0$}

\addplot {betapdf(x, 15.7, 7.3)};
\addlegendentry{$p=0.7,\ \nu=21.0$}

\addplot {betapdf(x, 73.9, 9.1)};
\addlegendentry{$p=0.9,\ \nu=81.0$}

\addplot {betapdf(x, 98.1375, 6.1125)};
\addlegendentry{$p=0.95,\ \nu=102.2$}

\end{axis}
\end{tikzpicture}
    }
    \caption{\small Adaptive domain informed prior over adjacency matrix}
    \label{fig:adaptive_domain_prior}
\end{wrapfigure}
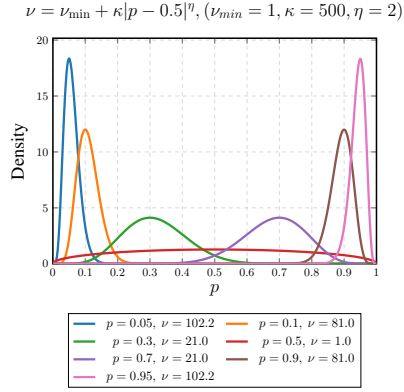
Marginalizing over \(\pi_{ij}\) recovers a Beta-Bernoulli model over \(B_{ij}\).
Since causal graphs are acyclic, we enforce acyclicity in the prior by construction~\citep{annadani2023bayesdag}.
We introduce a binary free edge matrix \(B\in\{0,1\}^{m\times m}\) with zero diagonal and a vector of node potentials \(r=(r_1,\dots, r_m)^T\in\R^m\). 
Let's define acyclic mask and induced adjacency by
\begin{equation*}
\label{eq:mask-and-adjacency}
M(r) = P(r)\,L\,P(r)^T,
\qquad
A = B \odot M(r),
\end{equation*}
where \(\odot\) denotes the Hadamard product. For any $(B,r)$, $A$ is a DAG (see Theorem \ref{thm:acyclic_construction} in appendix). Moreover, for any DAG $A$,  there always exists a corresponding pair $(B,r)$ such that $A = B \odot M(r)$ \citep{annadani2023bayesdag,yu2021dags}.
We therefore place joint prior on the free edges and potentials
$p(B,\Pi,r) \propto p(B,\Pi)\mathcal{N}(r\mid0,\sigma_r^2\Id)$.
This construction enforces acyclicity, so no additional acyclicity indicator or penalty is required.
The permutation \(P(r)\) is well defined when entries of \(r\) are pairwise distinct. Since \(p(r)\) is Gaussian, this condition holds almost surely.
By analytically marginalizing over $\Pi$ (see Derivation \ref{der:derivation_adjacency_prior} in appendix \ref{app:theory}), we have
\begin{equation*}
\label{eq:p_B_r_main}
p(B,r)
=\left(\prod_{i\neq j}\frac{\mathcal{B}\!\left(B_{ij}+\alpha_{ij},\,\beta_{ij}+1-B_{ij}\right)}{\mathcal{B}\!\left(\alpha_{ij},\beta_{ij}\right)}\right)\,
\mathcal{N}\!\left(r\mid0,\sigma_r^2\Id\right),
\end{equation*}
where \(\mathcal{B}\) denotes the Beta function. 
Under this formulation, \(p_{ij}\) is the nominal mode of the latent edge probability \(\pi_{ij}\), not the exact finite $\nu_{ij}$ marginal probability of $B_{ij}$. The latter is 
$\alpha_{ij}/(\alpha_{ij}+\beta_{ij})$ which approaches $p_{ij}$ as $\nu_{ij} \rightarrow \infty$.

\subsubsection{Prior over NCM Parameters conditioned on the Adjacency Matrix}

Next, conditional on $A$, we place a prior on NCM parameters $\theta = \{\theta_1,\dots,\theta_{m}\}$. We consider a prior
\[
p(\theta\mid A)
\;=\;
\mathcal{N}\!\Bigl(\theta\mid\ \mu_p(A),\ \mathrm{diag}\!\bigl(\sigma_p^2(A)\bigr)\Bigr)
\;=\;
\prod_{k=1}^{d_\theta}\mathcal{N}\!\bigl(\theta_k\mid\ \mu_{p,k}(A),\ \sigma_{p,k}^2(A)\bigr),
\]
 where $d_\theta$ is the dimensionality and $\mu_p(A) \in \mathbb{R}^{d_\theta}$ and $\sigma_p(A) \in (0,\infty)^{d_\theta}$ are deterministic functions of $A$.  In the simplest case, these functions are constant (for example, zero mean and fixed variance), but the notation permits dependence on \(A\).
The joint prior factorizes as
$
p(A, \theta)
\;=\;
p(\theta \mid A)\; p(A).
$

\subsubsection{ELBO without Normalizing Flow}
\noindent 
By marginalizing \(\Pi\) in the prior to get \(p(B, r)\) (see Derivation~\ref{der:derivation_adjacency_prior} in appendix \ref{app:theory}) and rearranging terms (see Derivation~\ref{prop:derivation_elbo} in appendix \ref{app:theory}), we have the ELBO as
\begin{equation*}
\begin{aligned}
\text{ELBO}(\phi)
&= \mathbb{E}_{q_\phi(B, r)} \Big[
      \underbrace{\mathbb{E}_{q_\phi(\theta \mid A)}
      [\log p(\mathcal{D} \mid A, \theta)]}_{\text{expected log likelihood}}
      - \underbrace{D_{\mathrm{KL}}\!\left(q_\phi(\theta \mid A) \,\|\, p(\theta \mid A)\right)}_{\text{KL divergence for }\theta}
   \Big] \\[4pt]
&\quad
 -\, \underbrace{D_{\mathrm{KL}}\!\left(q_\phi(B, r) \,\|\, p(B, r)\right)}_{\text{KL divergence for }B\text{ and }r}.
\end{aligned}
\label{eq:elbo-Bp}
\end{equation*}
where \(q_\phi(\theta \mid A)\) and \(q_\phi(B,r)\) denotes the conditional and structural guides.
\begin{restatable}{remark}{remark-ELBO}
Note that the above ELBO is not equivalent to the ELBO defined using \(q_\phi(A)\) under deterministic map \(A=B\odot M(r)\) because the pushforward from \((B,r)\) to \(A\) is many-to-one. If \(q_\phi(A)\) and \(p(A)\) are defined as the pushforwards of
\(q_\phi(B,r)\) and \(p(B,r)\), then Proposition~\ref{prop:data-processing} implies $\KL\!\bigl(q_\phi(B,r)\,\|\,p(B,r)\bigr)
\ge
\KL\!\bigl(q_\phi(A)\,\|\,p(A)\bigr)$.
Consequently, an ELBO written directly in terms of \(A\) is at least as tight as the ELBO in \((B,r)\) space. However, the induced prior \(p(A)\) is generally intractable, whereas \(p(B,r)\) admits a simple factorized form. We therefore work with Eq.~\ref{eq:ELBO-br} for computational tractability, even though ELBO in \(A\) space would be tighter.
\end{restatable}

\subsection{Soft Relaxation for Adjacency Matrix and Acyclicity}

Since $B$ is discrete, we cannot differentiate with respect to \(B_{ij}\). We therefore introduce a continuous relaxation for \(B\) via the Gumbel-Softmax trick \citep{jang2016categorical, maddison2016concrete, ng2022masked}. For each \((i,j)\) with \(i\neq j\), we introduce a learnable edge logit \(\gamma_{ij}\in\R\) and a temperature \(T>0\). We draw logistic noise $R_{ij}\ \stackrel{\mathrm{iid}}{\sim}\ \mathrm{Logistic}(0,1)$ implemented as $R_{ij}=\log U_{ij}-\log(1-U_{ij}),\ \ U_{ij}\stackrel{\mathrm{iid}}{\sim}\mathrm{Uniform}(0,1)$.
Equivalently, $R_{ij}$ can be generated as the difference of two i.i.d.\ \texttt{Gumbel}$(0,1)$ variables.
We use the same noise to define both a hard Bernoulli sample and its differentiable relaxation:
\begin{equation}
\label{eq:binary-concrete-sample}
B_{ij}=\mathbf{1}\{\gamma_{ij}+R_{ij}>0\},
\qquad
\tilde{B}_{ij}=\sigmoid\!\Bigl((\gamma_{ij}+R_{ij})/T\Bigr),
\qquad
\sigmoid(x)=1/(1+\mathrm{e}^{-x}),
\end{equation}
where \(B_{ii}=\tilde B_{ii}=0\). Then \(B_{ij}\sim\mathrm{Bernoulli}(\sigmoid(\gamma_{ij}))\), while \(\tilde{B}_{ij}\in(0,1)\) converges pointwise to \(B_{ij}\) as \(T\to0\), except on the probability-zero event \(\gamma_{ij}+R_{ij}=0\) \citep{ng2022masked}.
For the order, let \(P(r)\) be the hard permutation obtained by sorting \(r\). We also define a backward surrogate through the Sinkhorn operator \(\mathcal{S}\) \citep{adams2011ranking}. For \(\tau>0\), let \(S_0(r)=r\,\mathbf{o}^T\), where \(\mathbf{o}=(m,m-1,\dots,1)^T\), and applying \(K_S\) alternating row and column normalizations to \(\exp(S_0(r)/\tau)\) we have 
\begin{equation*}
\label{eq:P-and-M-tau}
P_{\tau,K_S}(r)=\mathcal S_{K_S}\!\left(\exp(S_0(r)/\tau)\right),
\qquad
M_{\tau,K_S}(r)=P_{\tau,K_S}(r)\,L\,P_{\tau,K_S}(r)^T.
\end{equation*}
Double stochasticity is guaranteed in the \(K_S\to\infty\) limit. More importantly, even an exactly doubly stochastic nonpermutation matrix generally makes \(M_{\tau,K_S}(r)\) dense, so this relaxed order mask is not generally acyclic and can have positive entries in both directions. Lemma~\ref{lem:sinkhorn-limit} (see appendix \ref{app:supporting_statements_and_proofs}) states that  \(P_{\tau,K_S}(r)\to P(r)\) and \(M_{\tau,K_S}(r)\to M(r)\) as \(\tau\to 0\) and \(K_S\to\infty\).
We distinguish the hard DAG from its soft backward surrogate as
\begin{equation}
\label{eq:hard-soft-st-adjacency}
A=B\odot M(r),\qquad
\tilde A=\tilde B\odot M_{\tau,K_S}(r)\odot(\Ones-\Id),\qquad
A_{\mathrm{ST}}=\tilde A+\operatorname{sg}(A-\tilde A),
\end{equation}
where \(M(r)=P(r)LP(r)^T\) and \(\operatorname{sg}\) is stop-gradient. Thus \(A_{\mathrm{ST}}=A\) in the forward pass and every evaluated graph is a binary DAG by Theorem~\ref{thm:acyclic_construction} (see appendix \ref{app:supporting_statements_and_proofs}). The diagonal factor removes self-loops from \(\tilde A\), but does not make the finite-temperature relaxation acyclic. 

\subsection{Forward Model}

\begin{wrapfigure}[12]{r}{0.4\textwidth}
    \vspace{-5em}   
    \centering
    \scalebox{0.38}{
        \begin{tikzpicture}[scale=1.2, >=stealth]

\pgfplotsset{compat=1.18}

\definecolor{C0}{HTML}{1f77b4}
\definecolor{C1}{HTML}{ff7f0e}
\definecolor{C2}{HTML}{2ca02c}
\definecolor{C3}{HTML}{d62728}
\definecolor{C4}{HTML}{9467bd}
\definecolor{C5}{HTML}{8c564b}
\definecolor{C6}{HTML}{e377c2}

\begin{axis}[
    width=10.5cm,
    height=7.8cm,
    xmin=-6, xmax=6,
    ymin=0, ymax=1.2,
    domain=-6:6,
    samples=600,
    smooth,
    grid=major,
    grid style={dashed, opacity=0.6},
    axis line style={line width=1.2pt},
    tick style={line width=1pt},
    ticklabel style={font=\small},
    xtick={-6,-4,...,6},
    ytick={0,0.2,...,1.2},
    xlabel={$\gamma$},
    ylabel={$p(\gamma)$},
    xlabel style={font=\Large},
    ylabel style={font=\Large},
    title={
        $\nu = \nu_{\min} + \kappa |p - 0.5|^{\eta},
        \ (\nu_{\min}=1,\ \kappa=500,\ \eta=2)$
    },
    title style={font=\Large, yshift=1ex},
    legend columns=2,
    legend cell align={left},
    legend style={
        at={(0.5,-0.22)},
        anchor=north,
        draw=black,
        font=\normalsize,
        row sep=2pt,
        column sep=6pt
    },
    cycle list={
        {C0, line width=1.8pt},
        {C1, line width=1.8pt},
        {C2, line width=1.8pt},
        {C3, line width=1.8pt},
        {C4, line width=1.8pt},
        {C5, line width=1.8pt},
        {C6, line width=1.8pt}
    },
    declare function={
        %
        %
        logisticbetapdf(\g,\a,\b,\lb) =
            exp(
                -\a*ln(1 + exp(-\g))
                -\b*ln(1 + exp(\g))
                -\lb
            );
    }
]

%
%

\addplot {
    logisticbetapdf(
        x,
        6.1125,
        98.1375,
        -23.20987847091072
    )
};
\addlegendentry{$p=0.05,\ \nu=102.2$}

\addplot {
    logisticbetapdf(
        x,
        9.1,
        73.9,
        -28.81585261438519
    )
};
\addlegendentry{$p=0.1,\ \nu=81.0$}

\addplot {
    logisticbetapdf(
        x,
        7.3,
        15.7,
        -14.24340032915957
    )
};
\addlegendentry{$p=0.3,\ \nu=21.0$}

\addplot {
    logisticbetapdf(
        x,
        1.5,
        1.5,
        -0.9347116558304359
    )
};
\addlegendentry{$p=0.5,\ \nu=1.0$}

\addplot {
    logisticbetapdf(
        x,
        15.7,
        7.3,
        -14.24340032915957
    )
};
\addlegendentry{$p=0.7,\ \nu=21.0$}

\addplot {
    logisticbetapdf(
        x,
        73.9,
        9.1,
        -28.81585261438519
    )
};
\addlegendentry{$p=0.9,\ \nu=81.0$}

\addplot {
    logisticbetapdf(
        x,
        98.1375,
        6.1125,
        -23.20987847091072
    )
};
\addlegendentry{$p=0.95,\ \nu=102.2$}

%

\addplot[
    C0,
    dotted,
    line width=1.8pt,
    forget plot
]
coordinates {
    (-2.9444389792,0)
    (-2.9444389792,1.2)
};

\addplot[
    C1,
    dotted,
    line width=1.8pt,
    forget plot
]
coordinates {
    (-2.1972245773,0)
    (-2.1972245773,1.2)
};

\addplot[
    C2,
    dotted,
    line width=1.8pt,
    forget plot
]
coordinates {
    (-0.8472978604,0)
    (-0.8472978604,1.2)
};

\addplot[
    C3,
    dotted,
    line width=1.8pt,
    forget plot
]
coordinates {
    (0,0)
    (0,1.2)
};

\addplot[
    C4,
    dotted,
    line width=1.8pt,
    forget plot
]
coordinates {
    (0.8472978604,0)
    (0.8472978604,1.2)
};

\addplot[
    C5,
    dotted,
    line width=1.8pt,
    forget plot
]
coordinates {
    (2.1972245773,0)
    (2.1972245773,1.2)
};

\addplot[
    C6,
    dotted,
    line width=1.8pt,
    forget plot
]
coordinates {
    (2.9444389792,0)
    (2.9444389792,1.2)
};

\addlegendimage{
    black,
    dotted,
    line width=1.8pt
}
\addlegendentry{$\operatorname{logit}(p)$}

\end{axis}
\end{tikzpicture}
    }
    \caption{\small Induced prior over edge logits}
    \label{fig:prior_edge_logits}
\end{wrapfigure}
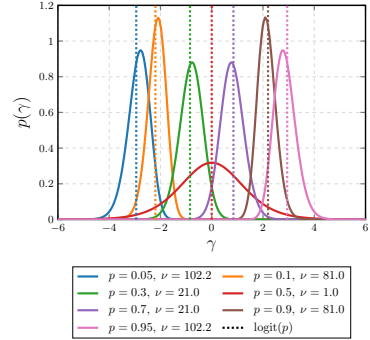

We now define forward model using Bayesian MLP for node \(i\) that uses the parent \(X_{\pa(i)}\) as input and outputs a prediction \(\hat{X}_i\). Formally,
$
\hat{X}_i \;=\; f_i\bigl(X_{\pa(i)};\,\theta_i\bigr).
$
To incorporate adjacency, let $X=(X_1,\dots,X_m)^T$ and let $(A_{\mathrm{ST}})_i$ denote the $i^\text{th}$ row of the straight-through matrix in Eq.~\ref{eq:hard-soft-st-adjacency}. 
We define
\[
g_i(x;\theta_i)=f_i\bigl((A_{\mathrm{ST}})_i\odot x;\theta_i\bigr).
\]
Since $A_{\mathrm{ST}}=A$ in the forward pass, the evaluated input zeros every nonparent exactly and all likelihood values correspond to a hard DAG. During backpropagation, derivatives are routed through the generally cyclic relaxation $\tilde A$. These straight through derivatives are biased surrogate gradients. 

\subsection{ELBO with Normalizing Flow}

\noindent In many applications, the assumption that $q_\phi(B)$ factorizes independently across edges can be restrictive. 
SVI \citep{rezende2015variational} addresses this limitation by constructing a more expressive \(q_\phi(B)\).
We define a standard base distribution \(p_0(z)\) for edge latent variables. We introduce an invertible conditional map $T_\phi: \mathbb{R}^{m(m-1)} \times \mathbb{R}^m \to \mathbb{R}^{m(m-1)}$, parameterized by $\phi$. The transformation gives $\gamma$ as a function of $z$ and $r$ given by $\gamma = T_\phi(z; r)$.
By rearrangement (see Derivation~\ref{prop:derivation_elbo_flow} in appendix \ref{app:theory}), we have the revised ELBO as
\begin{equation}
\label{eq:ELBO-flow-general-main}
\scalebox{0.85}{$
\begin{aligned}
\mathrm{ELBO}(\phi,\psi)
&=
\mathbb{E}_{r\sim q_\psi(r)}\;
\mathbb{E}_{z\sim p_0}\;
\mathbb{E}_{B\sim p(\cdot\mid T_\phi(z; r))}
\Bigl[
\underbrace{
\mathbb{E}_{q_\phi(\theta\mid A)}
\bigl[\log p(\mathcal{D}\mid A,\theta)\bigr]
}_{\text{expected log likelihood}}
-
\underbrace{
\KL\!\bigl(q_\phi(\theta\mid A)\,\|\,p(\theta\mid A)\bigr)
}_{\text{KL divergence for }\theta} \\
&\quad
+\ 
\log\!\Bigl|\det \frac{\partial T_\phi(z; r)}{\partial z}\Bigr|
-\log p_0(z)
+\log p\bigl(T_\phi(z; r)\bigr)
\Bigr]
\ -\ \underbrace{\KL\!\bigl(q_\psi(r)\,\|\,p(r)\bigr)}_{\text{KL divergence for potentials}}.
\end{aligned}
$}
\end{equation}
We define \(T_\phi\) as either a single flow transformation or a composition \(T_{\phi_K}^{(K)} \circ \cdots \circ T_{\phi_1}^{(1)}\). In compositional case, the log determinant term in above equation decomposes as
\begin{equation*}
\label{eq:normalizing_flow_sum}
\log\left|\det \frac{\partial T_\phi(z; r)}{\partial z}\right|=\sum_{k=1}^{K}\log\left|\det \frac{\partial T_{\phi_k}(z^{(k-1)};r)}{\partial z^{(k-1)}}\right|,\qquad z^{(k)}=T_{\phi_k}(z^{(k-1)};r),\ z^{(0)}=z.
\end{equation*}
Under Derivation~\ref{der:logistic-beta-prior} (see appendix \ref{app:theory}) and the modeling choice that model and guide share the hard conditional \(p(B\mid\gamma)\), the domain informed prior now enters through prior over edge logits \(p(\gamma)\) (see Figure \ref{fig:prior_edge_logits}) as
\begin{equation*}
\label{eq:logistic-beta-prior-main}
p(\gamma)
=
\prod_{i\neq j}
\frac{1}{\mathcal{B}(\alpha_{ij},\beta_{ij})}\,
\sigmoid(\gamma_{ij})^{\alpha_{ij}}\,
\bigl(1-\sigmoid(\gamma_{ij})\bigr)^{\beta_{ij}},
\qquad
\sigmoid(x)=1/(1+e^{-x}).
\end{equation*}
\noindent
\begin{restatable}{remark}{remark-prior}
Had we instead placed a Bernoulli prior $B_{ij}\sim\mathrm{Bernoulli}(p_{ij})$ on each edge, the edge probability would be deterministic (i.e., $\pi_{ij}=p_{ij}$). Equivalently, $p(\pi_{ij})=\delta(\pi_{ij}-p_{ij})$, and the logit transformation would induce $ p(\gamma_{ij}) = \delta\!\left(\gamma_{ij}-\operatorname{logit}(p_{ij})\right). $ Hence, a fixed Bernoulli prior induces a degenerate prior over $\gamma$ rather than an absolutely continuous density. This is not suitable to our formulation, since $q_\phi(\gamma\mid r)$ is continuous---so the KL divergence $\KL\!\left(q_\phi(\gamma\mid r)\,\|\,p(\gamma)\right)$ to this degenerate prior diverges unless the flow collapses to the same point mass. The Beta prior over $\pi_{ij}$ avoids this issue and gives the Logistic-Beta density. We further note that although the prior over \(\pi_{ij}\) is parameterized to have mode \(p_{ij}\), this mode is not preserved under the nonlinear logit transformation. In particular, \(\operatorname{logit}(p_{ij})\) is the image of the mode in \(\pi_{ij}\) space, whereas the mode of the Logistic-Beta density over \(\gamma_{ij}\) is $\log({\alpha_{ij}/\beta_{ij}})$. 
Consequently, the peaks of \(p(\gamma)\) need not align exactly with \(\operatorname{logit}(p_{ij})\), as seen in Figure~\ref{fig:prior_edge_logits}.
\end{restatable}

A single sample Monte Carlo (MC) estimator of revised ELBO is unbiased and practical in high dimension. 
The straight through backward pass is generally a biased gradient estimator.
If both \(q_\phi(\theta\mid A)\) and \(p(\theta\mid A)\) are diagonal Gaussians with strictly positive component scales, their KL divergence has the closed form (see Derivation \ref{der:KL-gaussian-diag} in appendix \ref{app:theory}) as
\begin{equation}
\label{eq:KL_closed_form}
\scalebox{0.99}{$
\KL\!\bigl(q_\phi(\theta\mid A)\,\|\,p(\theta\mid A)\bigr)
=\frac{1}{2}\sum_{k=1}^{d_\theta}\left[
\frac{\sigma_{\phi,k}^{2}(A)}{\sigma_{p,k}^{2}(A)}
+\frac{\bigl(\mu_{\phi,k}(A)-\mu_{p,k}(A)\bigr)^{2}}{\sigma_{p,k}^{2}(A)}
-1
+2\log\frac{\sigma_{p,k}(A)}{\sigma_{\phi,k}(A)}
\right].
$}
\end{equation} 

\subsection{Updating Node Potentials with SVGD using a Kernel defined in Relaxed Acyclicity Space}

Note that, even with a dependent guide, standard VI maximizes the ELBO, which is equivalent to minimizing the reverse KL divergence \(\KL(q_\phi \| p)\). For multimodal posteriors, this reverse KL objective can be mode-seeking~\citep{blei2017variational,zhang2018advances}.
In particular, it encourages the approximation to concentrate on a single mode, largely to avoid allocating mass to low-density posterior regions.
Consequently, although a mixture-of-Gaussians guide for $r$ could, in principle, represent multiple modes, standard ELBO optimization would still not guarantee recovery of all posterior modes. To promote mode coverage, we replace the parametric guide over \(r\) with a nonparametric particle approximation and update the particles using SVGD~\citep{liu2016stein}.
If we condition the revised ELBO (see Equation~\ref{eq:ELBO-flow-general-main}) on a fixed \(r\), we have
\begin{equation}
\scalebox{0.95}{$
\begin{aligned}
\mathcal L(\phi;r)
&=
\mathbb{E}_{z\sim p_0}\;
\mathbb{E}_{B\sim p(\cdot\mid T_\phi(z; r))}
\Bigl[
\underbrace{
\mathbb{E}_{q_\phi(\theta\mid A)}
\bigl[\log p(\mathcal{D}\mid A,\theta)\bigr]
}_{\text{expected log likelihood}}
-
\underbrace{
\KL\!\bigl(q_\phi(\theta\mid A)\,\|\,p(\theta\mid A)\bigr)
}_{\text{KL divergence for }\theta}\\
&\quad
+\ 
\log\!\Bigl|\det \frac{\partial T_\phi(z; r)}{\partial z}\Bigr|
-\log p_0(z)
+\log p\bigl(T_\phi(z; r)\bigr)
\Bigr].
\end{aligned}
$}
\label{eq:conditional-elbo}
\end{equation}
This constitutes a genuine lower bound on the conditional log evidence \(\mathcal L(\phi;r)\leq\log p(\mathcal D\mid r)\). More precisely, defining
$\Delta_\phi(r)=\KL\!\left(q_\phi(\gamma,B,\theta\mid r)\,\middle\|\,p(\gamma,B,\theta\mid\mathcal D,r)\right)$,
the standard variational identity gives \(\log p(\mathcal D\mid r)-\mathcal L(\phi;r)=\Delta_\phi(r)\geq0\). To introduce repulsive diversity pressure over node orders without assigning a density to a finite set of particles, we construct a proper target directly from the fixed \(r\) conditional bound. Specifically, for a fixed \(\phi\), we define
\begin{equation*}
\label{eq:tilted-target}
Z_\phi=\int p(r)\exp\{\mathcal L(\phi;r)\}\,dr,
\qquad
\tilde{p}_\phi(r\mid\mathcal D)=Z_\phi^{-1}p(r)\exp\{\mathcal L(\phi;r)\}.
\end{equation*}
If \(0<p(\mathcal D)<\infty\) and \(r\mapsto\mathcal L(\phi;r)\) is measurable, the nonnegativity \(\Delta_\phi(r)\ge 0\) implies \(0<Z_\phi\leq p(\mathcal D)\). Hence, $\tilde{p}_\phi(r\mid\mathcal D)$ constitutes a proper, lower-bound-induced target that relates to the exact potential posterior by $\tilde{p}_\phi(r\mid\mathcal D)=
p(r\mid\mathcal D)\exp\{-\Delta_\phi(r)\}/
{\mathbb E_{p(r\mid\mathcal D)}[\exp\{-\Delta_\phi(r)\}]}.$
Thus, \(\tilde{p}_\phi(r\mid\mathcal D)=p(r\mid\mathcal D)\) holds if and only if \(\Delta_\phi(r)\) is constant for \(p(r\mid\mathcal D)\) almost every \(r\) including, in particular, the special case of an exact conditional guide.
Assume that \(r\mapsto\mathcal L(\phi;r)\) is differentiable on the interior of each order chamber.
This requires the flow, likelihood, and parameter guide to be differentiable there and an integrable function that allows differentiation under the continuous expectations. The derivative of a sampled hard threshold is not this exact derivative. It would require an analytic finite sum or an unbiased discrete score estimator. Under these assumptions, away from tie hyperplanes, the score is $\nabla_r\log\tilde p_\phi(r\mid\mathcal D)
=\nabla_r\!\left[\mathcal L(\phi;r)+\log p(r)\right]$.
The hard permutation can make \(\mathcal L\) and \(\tilde{p}_\phi(r\mid\mathcal D)\) discontinuous across tie hyperplanes, even though these hyperplanes carry prior measure zero. We maintain a finite set \(\{r^{(k)}\}_{k=1}^{K_{\mathrm{part}}}\) purely as optimization variables. 
Holding these particles fixed, the exact conditional objective average that we optimize over \(\phi\) is $K_{\mathrm{part}}^{-1}\sum_{k=1}^{K_{\mathrm{part}}}\mathcal L(\phi;r^{(k)})$.
Notably, this is not a global ELBO for an empirical guide over \(r\).
Let's define the continuously differentiable, positive-definite Gaussian kernel $k_M(\psi,\psi')$  and let \(\psi(r)=\mathrm{vec}(M_{\tau,K_S}(r))\). 
Its pullback
$
k_r(r,r')=k_M(\psi(r),\psi(r'))
$
is, in turn, a continuously differentiable, positive-definite kernel on potential space. 
Applying the standard smooth-target SVGD formula to this score gives the particle update rule
\begin{equation}
\label{eq:svgd_update}
r^{(k)} \leftarrow r^{(k)} + \eta_r \frac{1}{K_\text{part}} \sum_{j=1}^{K_\text{part}} \left[ k_r(r^{(j)},r^{(k)}) \nabla_{r^{(j)}} \log \tilde{p}_\phi(r\mid\mathcal D) + \nabla_{r^{(j)}} k_r(r^{(j)},r^{(k)})\right].
\end{equation}
Here, \(\nabla_{r^{(j)}}k_r\) differentiates through \(\psi(r^{(j)})\). 
The second term in the bracket is the repulsive force that helps in mode coverage.

\subsection{Reparameterization Trick for Estimating Gradients}
In order to calculate gradients, let \(f_i\) be an \(L\)-layer MLP with parameters $\theta$ that contain weights and biases \(\{W^{(l)},b^{(l)}\}_{l=1}^{L}\) such that $W^{(l)}\in\mathbb{R}^{d_{\mathrm{out}}^{(l)}\times d_{\mathrm{in}}^{(l)}}$ and $b^{(l)}\in\mathbb{R}^{d_{\mathrm{out}}^{(l)}}$.
We approximate the conditional posterior over network parameters with a Gaussian guide. In software the hypernetwork $H_\phi$ receives \(A_{\mathrm{ST}}\) since its forward value is exactly \(A\), this evaluates the same density as \(q_\phi(\theta\mid A)\) while routing a surrogate derivative through \(\tilde A\):
\begin{equation*}
\label{eq:q-theta-hyper}
q_\phi(\theta\mid A_{\mathrm{ST}})
=
\mathcal{N}\!\Bigl(\theta \mid \mu_\phi(A_{\mathrm{ST}}),\ \mathrm{diag}\!\bigl(\sigma_\phi^2(A_{\mathrm{ST}})\bigr)\Bigr)
=
\prod_{k=1}^{d_\theta}\mathcal{N}\!\bigl(\theta_k \mid \mu_{\phi,k}(A_{\mathrm{ST}}),\ \sigma_{\phi,k}^2(A_{\mathrm{ST}})\bigr),
\end{equation*}
where $(\mu_\phi,\rho_\phi)=H_\phi(A_{\mathrm{ST}})$ and $\sigma_\phi=\mathrm{softplus}(\rho_\phi) = \log(1+\exp(\rho_\phi))$, which enforces positivity.
The conditioning is therefore explicit and differentiable. Changing $A_{\mathrm{ST}}$ changes the full variational posterior over $\theta$ through $H_\phi$.
\begin{equation*}
\theta=\mu_\phi(A_{\mathrm{ST}})+\sigma_\phi(A_{\mathrm{ST}})\odot\epsilon,
\qquad
\epsilon\sim\mathcal{N}(0,I_{d_\theta}),
\label{eq:theta-reparam-final}
\end{equation*}
which preserves differentiability through both \(H_\phi\) and the predictors \(f_i\). $\tilde{B}$ is reparameterized using the Gumbel-Sigmoid relaxation as shown in Eq.~\ref{eq:binary-concrete-sample}.
We refer the reader to Appendix \ref{app:pseudo_code} for the pseudo code and computational complexity of the \texttt{SVI-DAG} algorithm.

\section{Results and Discussion}

We evaluate \texttt{SVI-DAG} against five state-of-the-art (SOTA) Bayesian DAG learning methods on linear synthetic data, nonlinear synthetic data and real data using brier score, expected structural hamming distance (SHD), expected F1 score, and area under the receiver operating characteristic curve (AUROC). 
The rationale underlying the selection of these metrics is provided in appendix \ref{app:metrics}.
\paragraph{Datasets.} To isolate effect of domain informed prior, we use a two node graph with data generated from ANM with Gaussian noise, where small graph size makes role of prior tractable to analyze. For synthetic experiments at scale, we generate linear and nonlinear structural equation models over Erd\H{o}s--R\'enyi (ER) DAGs with 25 and 50 nodes. For real data, we evaluate on Sachs dataset. 

\paragraph{Baselines.} We focus our evaluation on Bayesian methods to DAG learning and compare with \texttt{ProDAG}~\citep{thompson2024prodag},
\texttt{BayesDAG}~\citep{annadani2023bayesdag}, \texttt{DDS}~\citep{charpentier2022differentiabledag}, \texttt{BCD} Nets~\citep{cundy2021bcd}, and \texttt{DiBS}~\citep{lorch2021dibs}. Appendix \ref{app:hyperparameters} describes the hyperparameters used for benchmarking these algorithms.

\begin{wrapfigure}{r}{0.3\textwidth}
    \centering
    \vspace{-10pt}
    \captionsetup{font=scriptsize}
    \includegraphics[width=\linewidth]{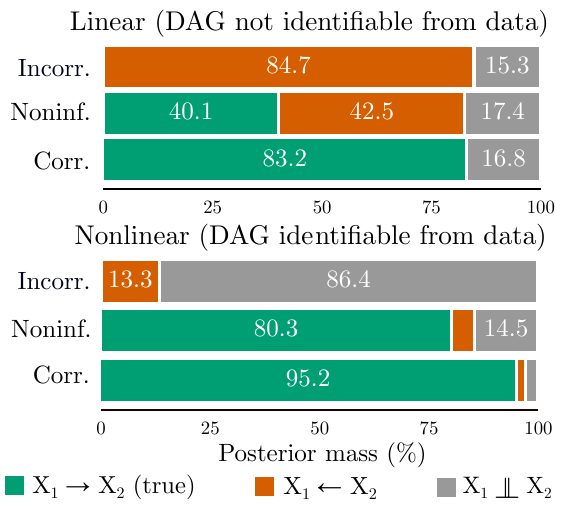}
    \vspace{-15pt}
    \caption{\scriptsize Effect of prior on 2 node graph}
    \label{fig:domain_prior_effect}
    \vspace{-10pt}
\end{wrapfigure}


\subsection{Effect of Domain Informed Prior}

Figure \ref{fig:domain_prior_effect} illustrates effect of prior on a 2 node graph using 1000 samples, chosen to balance influence of likelihood and prior. After optimization, we draw $10000$ hard posterior DAG samples $A = B \odot M(r)$, $r$ sampled uniformly from the SVGD particles, $\gamma$ from $q(\gamma \mid r)$, hard Bernoulli edges $B = \mathbf{1}\{\gamma + R > 0\}$ and the hard permutation mask $M(r) = P(r)\,L\,P(r)^{\top}$, so every sample is a binary DAG and report the proportion of samples assigned to \(x_1 \to x_2\), \(x_1 \gets x_2\), and \(x_1 \perp\!\!\!\perp x_2\). With ground truth \(x_1 \to x_2\), linear Gaussian setting (non-identifiable) shows that prior dominates orientation:
an incorrect prior (p = 0.01 correct, 0.99 incorrect) concentrates mass on the wrong direction (84.7\%), a noninformative prior (p = 0.5) gives balanced posterior, and correct prior (p = 0.99 correct, 0.01 incorrect) shifts mass to true direction (83.2\%). In nonlinear Gaussian setting (identifiable), prior acts primarily as regularizer: the incorrect prior suppresses both edges, concentrating 86.4\% on \(x_1 \perp\!\!\!\perp x_2\), while a correct prior sharpens posterior and assigns highest mass to \(x_1 \to x_2\) (95.2\%).

\subsection{Linear Synthetic Data}

\begin{figure*}[ht]
\centering
\includegraphics[width=\linewidth]{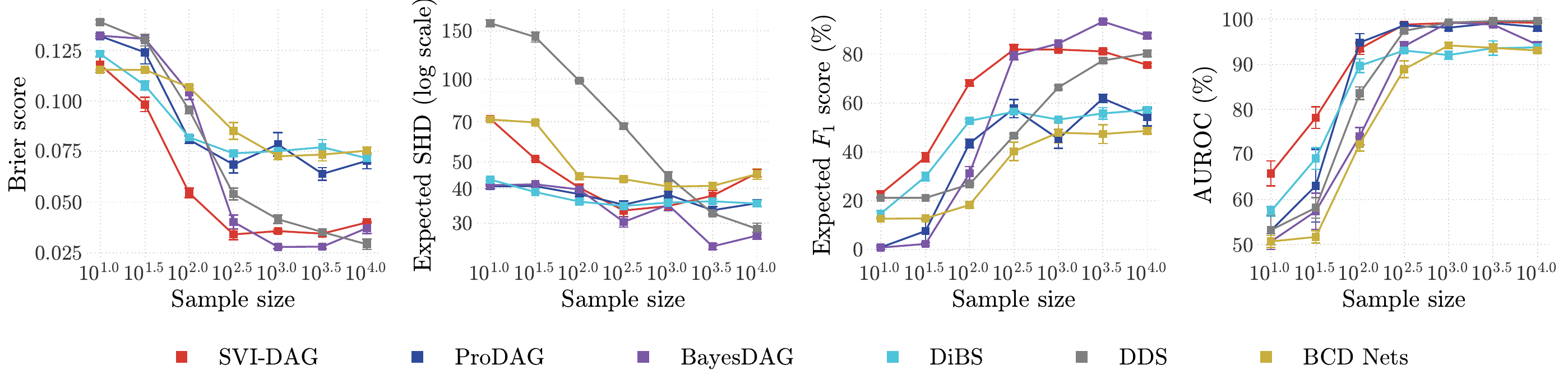}
\caption{Results on synthetic data generated from linear ER DAGs with $p=25$ nodes and $s=40$ edges under Gaussian noise. Each point denotes sample mean, and accompanying error bars indicate standard error, both computed across 5 independent draws of the dataset.}
\label{fig:synthetic-linear-25-40-erdos-renyi-gaussian}
\end{figure*}

\begin{figure*}[ht]
\centering
\includegraphics[width=\linewidth]{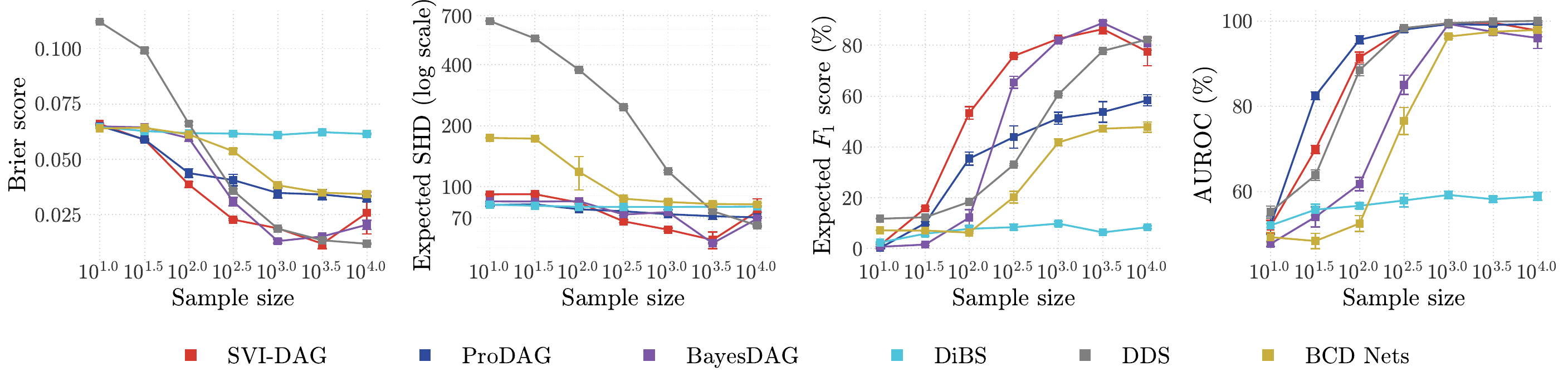}
\caption{Results on synthetic data generated from linear ER DAGs with $p=50$ nodes and $s=80$ edges under Gaussian noise. Each point denotes sample mean, and accompanying error bars indicate standard error, both computed across 5 independent draws of the dataset.}
\label{fig:synthetic-linear-50-80-erdos-renyi-gaussian}
\end{figure*}

Figures \ref{fig:synthetic-linear-25-40-erdos-renyi-gaussian} and \ref{fig:synthetic-linear-50-80-erdos-renyi-gaussian} present results on datasets generated from linear ER DAGs with $p=25$, $s=40$ and $p=50$, $s=80$ across increasing sample sizes. Since the underlying structure is identifiable only up to its MEC, all metrics are evaluated at CPDAG level. \texttt{SVI-DAG} achieves lower Brier scores and higher E-F1 score across most of the sample sizes while remaining competitive with baselines in terms of E-SHD and AUROC.

\subsection{Nonlinear Synthetic Data}

\begin{figure*}[ht]
\centering
\includegraphics[width=\linewidth]{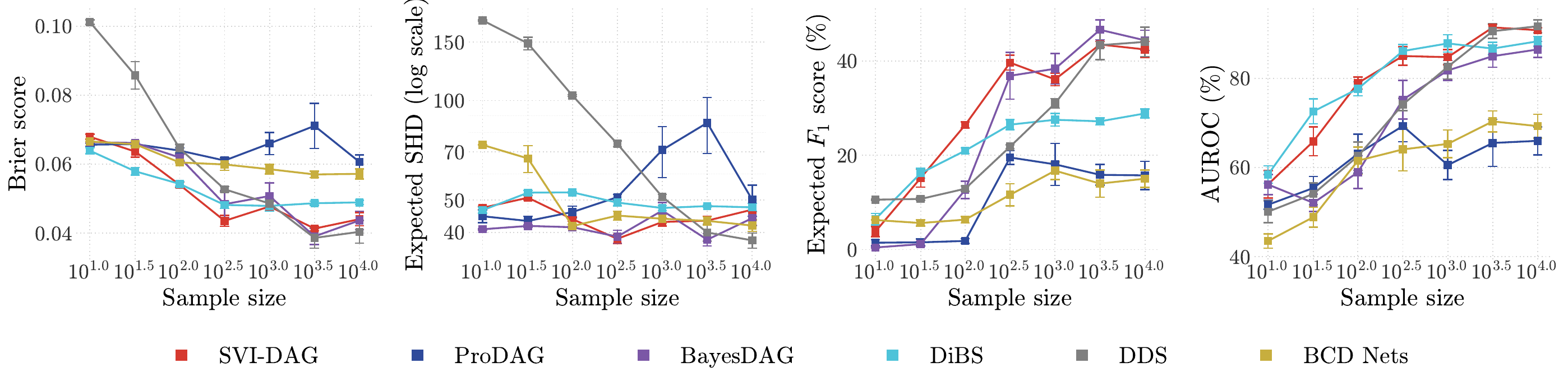}
\caption{Results on synthetic data generated from nonlinear ER DAGs with $p=25$ nodes and $s=40$ edges under Gaussian noise. Each point denotes the sample mean, and the accompanying error bars indicate the standard error, both computed across 5 independent draws of the dataset.}
\label{fig:synthetic-nonlinear-25-40-erdos-renyi-gaussian}
\end{figure*}

\begin{figure*}[ht]
\centering
\includegraphics[width=\linewidth]{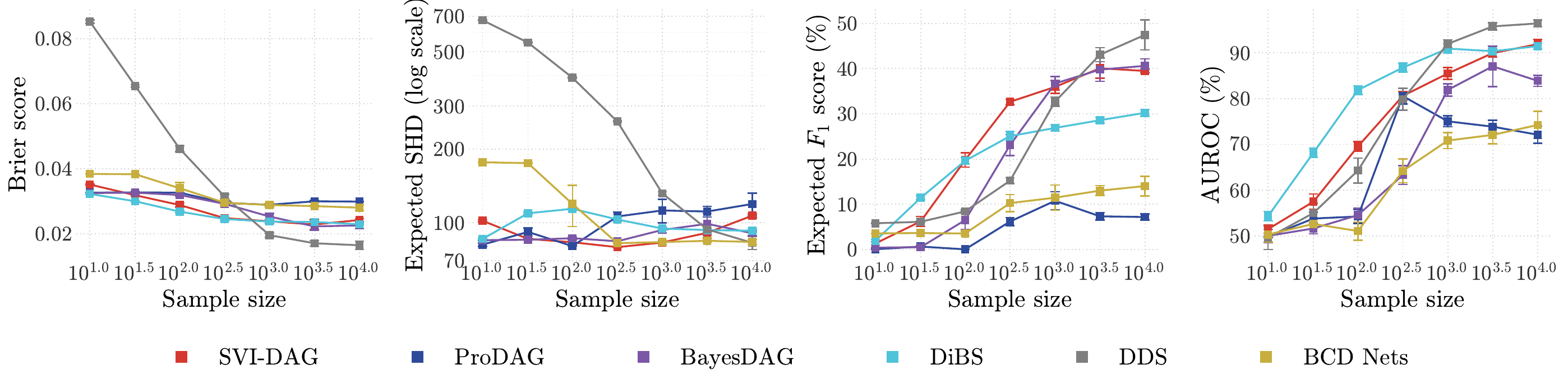}
\caption{Results on synthetic data generated from nonlinear ER DAGs with $p=50$ nodes and $s=80$ edges under Gaussian noise. Each point denotes sample mean, and accompanying error bars indicate standard error, both computed across 5 independent draws of the dataset.}
\label{fig:synthetic-nonlinear-50-80-erdos-renyi-gaussian}
\end{figure*}

Figures \ref{fig:synthetic-nonlinear-25-40-erdos-renyi-gaussian} and \ref{fig:synthetic-nonlinear-50-80-erdos-renyi-gaussian} present results on datasets generated from nonlinear ER DAGs with $p=25$, $s=40$ and $p=50$, $s=80$ across increasing sample sizes. In this setting, the underlying DAG is identifiable, and all metrics are therefore evaluated at the DAG level. We can observe that \texttt{SVI-DAG} attains a lower Brier score and higher E-F1 across most sample sizes, while also staying competitive with other algorithms in terms of the other metrics.

\subsection{Real Data}

\begin{wraptable}[10]{r}{0.35\textwidth}
\vspace{-1.2em}
\centering
\captionsetup{font=scriptsize}
\caption{\scriptsize Wall-clock time per train/evaluate on the Sachs dataset on an A100-40GB GPU using JAX.}
\label{tab:sachs_runtime}
\small
\begin{tabular}{@{}lr@{}}
\toprule
 & Time per run (s) \\
\midrule
\texttt{SVI-DAG} & 134.2 $\pm$ 30.9 \\
\texttt{ProDAG} & 100.2 $\pm$ 4.1 \\
\texttt{BayesDAG} & 266.5 $\pm$ 13.0 \\
\texttt{DDS} & 65.1 $\pm$ 1.9 \\
\texttt{BCD Nets} & 151.1 $\pm$ 2.5 \\
\texttt{DiBS} & \textbf{5.6 $\pm$ 0.7} \\
\bottomrule
\end{tabular}
\vspace{-1em}
\end{wraptable}

\begin{table*}[ht]
\centering
\caption{Performance on Sachs dataset. The average and standard errors are measured over 10-fold cross-validation splits of data. The best value of each DAG metric is indicated in bold.}
\label{tab:sachs-dag}
\small
\begin{tabularx}{\linewidth}{XXXXX}
\toprule
 & Brier score ($\downarrow$) & Exp. SHD ($\downarrow$) & Exp. F1 score ($\uparrow$) & AUROC ($\uparrow$) \\
\midrule
\texttt{SVI-DAG} & 0.185 $\pm$ 0.007 & 25.70 $\pm$ 0.92 & \textbf{28.95 $\pm$ 1.73} & 58.16 $\pm$ 2.27 \\
\texttt{ProDAG} & 0.159 $\pm$ 0.003 & \textbf{19.68 $\pm$ 0.35} & 14.32 $\pm$ 1.54 & 52.95 $\pm$ 1.81 \\
\texttt{BayesDAG} & 0.174 $\pm$ 0.008 & 28.66 $\pm$ 0.78 & 26.75 $\pm$ 2.08 & 62.23 $\pm$ 1.93 \\
\texttt{DDS} & 0.168 $\pm$ 0.004 & 24.62 $\pm$ 0.56 & 23.21 $\pm$ 1.52 & 51.35 $\pm$ 2.18 \\
\texttt{BCD Nets} & \textbf{0.153 $\pm$ 0.004} & 20.94 $\pm$ 0.50 & 10.79 $\pm$ 2.16 & 55.91 $\pm$ 2.07 \\
\texttt{DiBS} & 0.180 $\pm$ 0.004 & 40.74 $\pm$ 0.50 & 27.78 $\pm$ 0.74 & \textbf{65.63 $\pm$ 2.07} \\
\bottomrule
\end{tabularx}
\end{table*}

The flow cytometry dataset~\citep{sachs2005causal} is a standard benchmark for
causal discovery in biological systems, comprising $7466$ samples over $11$
variables and $18$ edges corresponding to phosphoproteins and phospholipids.
Tables \ref{tab:sachs-dag} and \ref{tab:sachs-cpdag} report the performance on
the Sachs dataset under DAG and CPDAG metrics, respectively. We observe that
$\texttt{SVI-DAG}$ consistently outperforms the competing algorithms in F1
score for the DAG metrics and in AUROC for the CPDAG metrics. Table
\ref{tab:sachs_runtime} further reports the computational time of the algorithms
on this dataset, where $\texttt{SVI-DAG}$ remains comparable to the other
algorithms.

We note that all results reported for \texttt{SVI-DAG} in this paper correspond
to a single choice of normalizing flow---neural spline flows. Our algorithm is
deliberately flow-agnostic, and hence any normalizing flow family can be
substituted for the spline flow without modifying the rest of the framework. The
reported results should therefore be read as the performance of one particular
instantiation of \texttt{SVI-DAG}, rather than an upper bound on what the
approach can ultimately achieve.

\begin{table*}[ht]
\centering
\caption{Performance on Sachs dataset. The average and standard errors are measured over 10-fold cross-validation splits of data. The best value of each CPDAG metric is indicated in bold.}
\label{tab:sachs-cpdag}
\small
\begin{tabularx}{\linewidth}{XXXXX}
\toprule
 & Brier score ($\downarrow$) & Exp. SHD ($\downarrow$) & Exp. F1 score ($\uparrow$) & AUROC ($\uparrow$) \\
\midrule
\texttt{SVI-DAG} & 0.237 $\pm$ 0.016 & 23.08 $\pm$ 0.87 & 54.05 $\pm$ 2.07 & \textbf{70.17 $\pm$ 2.28} \\
\texttt{ProDAG} & 0.267 $\pm$ 0.001 & \textbf{17.36 $\pm$ 0.07} & 31.55 $\pm$ 0.64 & 62.70 $\pm$ 0.90 \\
\texttt{BayesDAG} & 0.247 $\pm$ 0.010 & 26.38 $\pm$ 0.42 & 50.10 $\pm$ 1.36 & 62.88 $\pm$ 1.42 \\
\texttt{DDS} & 0.240 $\pm$ 0.005 & 23.02 $\pm$ 0.29 & 43.46 $\pm$ 0.91 & 68.91 $\pm$ 1.30 \\
\texttt{BCD Nets} & 0.259 $\pm$ 0.002 & 18.68 $\pm$ 0.16 & 26.30 $\pm$ 1.61 & 56.65 $\pm$ 1.62 \\
\texttt{DiBS} & \textbf{0.215 $\pm$ 0.003} & 35.49 $\pm$ 0.22 & \textbf{54.46 $\pm$ 0.29} & 66.80 $\pm$ 1.10 \\
\bottomrule
\end{tabularx}
\end{table*}

\subsection{Ablation Studies}

\begin{table}[ht]
\centering
\caption{Component ablation. Nonlinear ER $p=20$, $s=40$, mean $\pm$ s.e.\ over 10 seeds. MEC-cov: expected fraction of the true Markov equivalence class receiving posterior mass from the separate linear study ($p=10$, $s=10$).}
\label{tab:ablation}
\small
\begin{tabular}{lccccc}
\toprule
Variant & Brier score $\downarrow$ & Exp. SHD $\downarrow$ & Exp. F1 (\%) $\uparrow$ & AUROC (\%) $\uparrow$ & MEC-cov $\uparrow$ \\
\midrule
\texttt{SVI-DAG} & \textbf{0.077 $\pm$ 0.002} & 69.3 $\pm$ 1.1 & 26.9 $\pm$ 0.8 & \textbf{81.7 $\pm$ 1.2} & \textbf{0.06 $\pm$ 0.02} \\
(a) w/o flow & 0.078 $\pm$ 0.002 & 66.0 $\pm$ 0.8 & 25.4 $\pm$ 0.7 & 78.9 $\pm$ 1.2 & 0.00 $\pm$ 0.00 \\
(b) w/o SVGD & 0.083 $\pm$ 0.006 & \textbf{62.0 $\pm$ 2.5} & \textbf{34.8 $\pm$ 2.1} & 73.7 $\pm$ 3.2 & 0.05 $\pm$ 0.04 \\
(c) w/o domain prior & 0.208 $\pm$ 0.003 & 164.4 $\pm$ 1.5 & 18.2 $\pm$ 0.4 & 57.0 $\pm$ 2.4 & 0.00 $\pm$ 0.00 \\
(d) w/o flow \& SVGD & 0.083 $\pm$ 0.005 & 62.2 $\pm$ 2.0 & 29.1 $\pm$ 2.1 & 70.3 $\pm$ 3.7 & 0.00 $\pm$ 0.00 \\
\bottomrule
\end{tabular}
\end{table}

We assess the contribution of each component of \texttt{SVI-DAG} by ablating them. Table~\ref{tab:ablation} summarizes the results. We first observe that no ablation improves edge ranking upon the full model: AUROC is highest for the full model and degrades in every ablated variant. The prior is the dominant component in this regime. Specifically, with $p(p-1)=380$ candidate edges informed by only $300$ observations, removing the sparsity prior more than doubles the expected SHD and degrades calibration.
To ablate SVGD, we replace the particle ensemble over order potentials with a reparameterized Gaussian guide. Observe that a unimodal guide can commit to a single ordering mode. Hence we can observe relatively lower value for AUROC and higher value of Brier. The role of the flow requires more care, and we treat it in two parts. Observe
first that every posterior DAG factorizes as $A = B \odot M(r)$. Consequently, brier, AUROC, and expected SHD are functionals of the posterior edge marginals alone, and a mean-field $q(\gamma)$ can in principle match any marginal vector the conditional flow induces. The close agreement between the first two rows in the calibration column is therefore expected, and should not be read as evidence that the flow is inert. The flow's contribution appears decisively in the final column. In a companion
study on linear graphs small enough that the Markov equivalence class can be
enumerated exactly, only the flow bearing variants achieve nonzero coverage of
the true class. Removing the flow thus eliminates exact structure recovery entirely. Intuitively, the probability assigned to a specific graph is a product over hundreds of per edge events, and only a correlated posterior can concentrate mass on coherent whole
graphs. 

\section{Conclusion}
\label{sec:conclusion}

We presented \texttt{SVI-DAG}, a differentiable Bayesian approach to causal discovery that uses normalizing flows to encode dependencies between edges for expressive and multimodal learning over DAGs. 

\section{Acknowledgment}
\label{sec:acknowledgment}

The author gratefully acknowledges the computing resources provided by Purdue University, and thanks Prof. Ilias Bilionis for access to these resources and for the freedom to pursue this line of work during the author's doctoral studies.

\bibliographystyle{plainnat}
\bibliography{references}

\newpage
\appendix
\numberwithin{equation}{section}
\numberwithin{figure}{section}
\numberwithin{table}{section}

{\LARGE \textbf{Appendix}}

This appendix includes additional details for the paper, ``SVI-DAG: A Structured Variational Inference Approach to Bayesian Causal Discovery".

\section{Theory}
\label{app:theory}

\subsection{Derivation of $p(B, r)$}
\label{der:derivation_adjacency_prior}
We start from the construction that enforces acyclicity by mapping the free edge matrix $B\in\{0,1\}^{m\times m}$ and the potentials $r\in\mathbb{R}^m$ to the adjacency $A=B\odot M(r)$ where $M(r)$ is the acyclic mask induced by the order defined by $r$. 
The joint prior over $B$, the latent edge probabilities $\Pi=\{\pi_{ij}\}_{i\neq j}$, and $r$ is
\begin{equation*}
p(B,\Pi,r)
\;\propto\;
\left(\prod_{i\neq j}\operatorname{Bernoulli}\!\left(B_{ij}\mid\pi_{ij}\right)\,
\operatorname{Beta}\!\left(\pi_{ij}\mid\alpha_{ij},\beta_{ij}\right)\right)\,
\mathcal{N}\!\left(r\mid0,\sigma_r^2\Id\right),
\end{equation*}
with $B_{ii}=0$ for all $i$ and with independence between $r$ and the set of free edges in the prior. The factors are
\begin{equation*}
\operatorname{Bernoulli}\!\left(B_{ij}\mid\pi_{ij}\right)=\pi_{ij}^{B_{ij}}\left(1-\pi_{ij}\right)^{1-B_{ij}},
\qquad
\operatorname{Beta}\!\left(\pi_{ij}\mid\alpha_{ij},\beta_{ij}\right)
=\frac{\pi_{ij}^{\alpha_{ij}-1}\left(1-\pi_{ij}\right)^{\beta_{ij}-1}}{\mathcal{B}\!\left(\alpha_{ij},\beta_{ij}\right)},
\end{equation*}
where $\mathcal{B}$ denotes the Beta function. Our goal is to integrate out $\pi_{ij}$ to get the marginal prior on $(B,r)$. For a fixed ordered pair $(i,j)$ with $i\neq j$ we compute
\begin{equation*}
p(B_{ij})
=\int_{0}^{1}\operatorname{Bernoulli}\!\left(B_{ij}\mid\pi_{ij}\right)\,
\operatorname{Beta}\!\left(\pi_{ij}\mid\alpha_{ij},\beta_{ij}\right)\,d\pi_{ij}.
\end{equation*}
Substituting the explicit forms gives
\begin{equation*}
p(B_{ij})
=\frac{1}{\mathcal{B}\!\left(\alpha_{ij},\beta_{ij}\right)} 
\int_{0}^{1}\pi_{ij}^{B_{ij}+\alpha_{ij}-1}\left(1-\pi_{ij}\right)^{\left(1-B_{ij}\right)+\beta_{ij}-1}\,d\pi_{ij}.
\end{equation*}
The integral equals a Beta function with updated parameters, which gives the standard Beta Bernoulli marginal
\begin{equation*}
p(B_{ij})=\frac{\mathcal{B}\!\left(B_{ij}+\alpha_{ij},\,\beta_{ij}+1-B_{ij}\right)}{\mathcal{B}\!\left(\alpha_{ij},\beta_{ij}\right)}.
\end{equation*}
Using independence across ordered pairs in the prior and independence of $r$ from $B$ in the prior, we get the full marginal as
\begin{equation*}
\label{eq:p_B_r}
p(B,r)
=\left(\prod_{i\neq j}\frac{\mathcal{B}\!\left(B_{ij}+\alpha_{ij},\,\beta_{ij}+1-B_{ij}\right)}{\mathcal{B}\!\left(\alpha_{ij},\beta_{ij}\right)}\right)\,
\mathcal{N}\!\left(r\mid 0,\sigma_r^2\Id\right).
\end{equation*}
The diagonal constraint $B_{ii}=0$ excludes self loops and is implemented by masking the diagonal of $B$. 
Note that instead of $p(B_{ij}=1)=p_{ij}$ and $p(B_{ij}=0)=1 - p_{ij}$ we have
$$
p(B_{ij}=1)=\frac{\alpha_{ij}}{\alpha_{ij}+\beta_{ij}},
\qquad
p(B_{ij}=0)=\frac{\beta_{ij}}{\alpha_{ij}+\beta_{ij}},
$$
for a chosen finite $\nu_{ij}$. This completes the derivation.

\subsection{Derivation of ELBO without the normalizing flow}
\label{prop:derivation_elbo}
Given \(A=B\odot M(r)\), which is a deterministic function of \((B,r)\) alone, we have the full joint model as $p(B,\Pi,r,\theta,\mathcal{D}) = p(B,\Pi,r)\;p(\theta\mid A)\;p(\mathcal{D}\mid A,\theta)$.
The posterior is
\[
p(B,\Pi,r,\theta\mid \mathcal{D}) 
= \frac{p(B,\Pi,r,\theta,\mathcal{D})}{\int p(B,\Pi,r,\theta,\mathcal{D})\, dB\,d\Pi\,dr\,d\theta}.
\]
\noindent
Since this posterior is intractable, we approximate it with a variational distribution
\[
q_\phi(B,\Pi, r, \theta) = q_\phi(B,\Pi, r)\;q_\phi(\theta \mid A), \qquad A=B\odot M(r).
\]
\noindent 
Here \(\phi\) denotes the collection of variational parameters for \(B\), \(\Pi\), \(r\), and \(\theta\). We begin by writing the evidence as
$
\log p(\mathcal{D}) = \log \int p(B,\Pi,r,\theta,\mathcal{D})\,dB\,d\Pi\,dr\,d\theta.
$
By definition, the ELBO is
$$
\text{ELBO}(\phi)=\mathbb{E}_{q_\phi(B, \Pi, r, \theta)}\left[\log p(B, \Pi, r, \theta, \mathcal{D})-\log q_\phi(B, \Pi, r, \theta)\right].
$$
Using $p(B,\Pi,r,\theta,\mathcal{D}) = p(\mathcal{D} \mid A,\theta)\;p(\theta\mid A)\;p(B,\Pi,r)$, we have
\begin{equation*}
\text{ELBO}(\phi)=\mathbb{E}_{q_\phi(B, \Pi, r, \theta)}\!\left[\log p(\mathcal{D} \mid A, \theta)+\log p(\theta \mid A)+\log p(B, \Pi, r)-\log q_\phi(B, \Pi, r, \theta)\right].
\end{equation*}
If we assume independence between structural variables and parameter variables in the guide, we can factorize:
\begin{equation*}
q_\phi(B, \Pi, r, \theta)=q_\phi(B, \Pi, r)\, q_\phi(\theta \mid A),
\end{equation*}
so that
\begin{equation*}
\log q_\phi(B, \Pi, r, \theta)=\log q_\phi(B, \Pi, r)+\log q_\phi(\theta \mid A).
\end{equation*}
By linearity of expectation and the definition of the KL divergence ($\KL(q\|p)=\mathbb{E}_q[\log q-\log p]$), the ELBO can then be rearranged as:
\begin{equation*}
\begin{aligned}
\label{eq:elbo_simple}
\text{ELBO}(\phi)
&=\mathbb{E}_{q_\phi(B, \Pi, r)}\!\Big[
      \mathbb{E}_{q_\phi(\theta \mid A)}\!\left[
         \log p(\mathcal{D} \mid A, \theta)
         + \log p(\theta \mid A)
      \right] \\
&\qquad\qquad\qquad
      -\,\mathbb{E}_{q_\phi(\theta \mid A)}\!\left[
         \log q_\phi(\theta \mid A)
      \right]
      + \log p(B, \Pi, r)
      - \log q_\phi(B, \Pi, r)
   \Big] \\[6pt]
&=\mathbb{E}_{q_\phi(B, \Pi, r)}\!\Big[
      \mathbb{E}_{q_\phi(\theta \mid A)}
      [\log p(\mathcal{D} \mid A, \theta)]
   \Big] \\
&\qquad
   -\,\mathbb{E}_{q_\phi(B, \Pi, r)}\!\Big[
      \KL\!\left(q_\phi(\theta \mid A) \,\|\, p(\theta \mid A)\right)
   \Big]
   - \KL\!\left(q_\phi(B, \Pi, r) \,\|\, p(B, \Pi, r)\right).
\end{aligned}
\end{equation*}
If we analytically marginalize $\Pi$ in the prior to get $p(B, r)$ (see Derivation~\ref{der:derivation_adjacency_prior}) we have
\begin{equation*}
\label{eq:prior_B_r_without_normalizing_flow}
p(B, r)=\left(\prod_{i \neq j} \frac{\mathcal{B}\left(B_{i j}+\alpha_{i j}, \beta_{i j}+1-B_{i j}\right)}{\mathcal{B}\left(\alpha_{i j}, \beta_{i j}\right)}\right) \,\mathcal{N}(r\mid0,\sigma_r^2\Id),
\end{equation*}
\noindent 
where $\mathcal{B}$ represents the Beta function.
Since $\Pi$ is an auxiliary latent variable, marginalizing it out does not change the evidence $p(\mathcal{D})$. We therefore perform variational inference directly on the collapsed model
$$
p(\mathcal{D},B,r,\theta)=p(\mathcal{D}\mid A,\theta)\,p(\theta\mid A)\,p(B,r),
\qquad
A=B\odot M(r),
$$
and the corresponding guide $q_\phi(B,r)\,q_\phi(\theta\mid A)$. The ELBO of this collapsed model is
\begin{equation}
\label{eq:ELBO-br}
\begin{aligned}
\text{ELBO}(\phi)
&= \mathbb{E}_{q_\phi(B, r)} \Big[
      \underbrace{\mathbb{E}_{q_\phi(\theta \mid A)}
      [\log p(\mathcal{D} \mid A, \theta)]}_{\text{expected log likelihood}}
      - \underbrace{D_{\mathrm{KL}}\!\left(q_\phi(\theta \mid A) \| p(\theta \mid A)\right)}_{\text{KL divergence for }\theta}
   \Big] \\
&\quad
   - \underbrace{D_{\mathrm{KL}}\!\left(q_\phi(B, r) \| p(B, r)\right)}_{\text{KL divergence for }B\text{ and }r}.
\end{aligned}
\end{equation}

\subsection{Derivation of ELBO with Normalizing Flow}
\label{prop:derivation_elbo_flow}

We derive the ELBO when the edge logits $\gamma$ are generated by a conditional normalizing flow $T_\phi$ and acyclicity is enforced by construction. The construction uses a free edge matrix $B$ with zero diagonal and a vector of node potentials $r\in\mathbb{R}^m$. The adjacency matrix used inside the model is given by
\begin{equation*}
\label{eq:A-hard-flow}
A=B\odot M(r),
\qquad
M(r)=P(r)\,L\,P(r)^T,
\end{equation*}
where $P(r)$, $L$, and $M(r)$ are as defined in the main text.
Note that a factorization where $q_\phi(A)$ assumes independence between $r$ and $\gamma$ prevents the model from capturing multiple modes in the DAG posterior. For example, if two DAGs fit the data equally well but imply different topological orders, an independent guide $q_\phi(r)\,q_\phi(\gamma)$ forces the model to average edge probabilities across incompatible orders, which can give a dense graph that does not coincide with any high posterior DAG. So to represent multiple modes, we model dependence between $r$ and $\gamma$.
Let \(z\sim p_0(z)\) and \(\gamma=T_\phi(z;r)\), where for each \(r\), \(T_\phi(\cdot;r)\) is a \(C^1\) diffeomorphism. Assume \(0<p(\mathcal D)<\infty\), \(q_\psi(r)\ll p(r)\), \(q_\phi(\gamma\mid r)\ll p(\gamma)\), \(q_\phi(\theta\mid A)\ll p(\theta\mid A)\), and integrability under \(q_{\phi,\psi}\) of all displayed log-density terms, including the log likelihood. The induced conditional guide over \(\gamma\) is given by the change of variables
\begin{equation*}
\label{eq:q-gamma-flow}  
q_\phi(\gamma \mid r)
= 
p_0 \bigl(T_\phi^{-1}(\gamma; r)\bigr)
\,
\left|\det \frac{\partial T_\phi^{-1}(\gamma; r)}{\partial \gamma}\right|
\;=\;
p_0(z)\,
\left|\det \frac{\partial T_\phi(z; r)}{\partial z}\right|^{-1}.
\end{equation*}
\noindent The hard DAG adjacency is a deterministic function of \(B\) and \(r\). The joint model is
\begin{equation}
\label{eq:model-joint}
p(\mathcal{D},B,r,\theta,\gamma)
=
p(\mathcal{D}\mid A,\theta)\,
p(\theta\mid A)\,
p(B\mid \gamma)\, p(r)\,
p(\gamma),
\end{equation}
where \(p(B\mid\gamma)\) is the Bernoulli conditional and \(p(r)\) is the potential prior.  Note that while the guide $q_\phi$ conditions $\gamma$ on $r$, the prior $p(r)p(\gamma)$ typically factorizes unless specific domain knowledge suggests otherwise.
We have the joint guide as
\begin{equation}
\label{eq:guide-factorization}
q_{\phi,\psi}(\gamma,B,r,\theta)
=
q_\phi(\gamma \mid r)\,
p(B\mid \gamma)\,
q_\psi(r)\,
q_\phi(\theta\mid A),
\end{equation}
to avoid the collapse to a single mode over DAGs. By definition,
\begin{equation*}
\label{eq:ELBO-def}
\mathrm{ELBO}(\phi,\psi)
=
\mathbb E_{q_{\phi,\psi}(\gamma,B,r,\theta)}
\Bigl[
\log p(\mathcal{D},B,r,\theta,\gamma)
-
\log q_{\phi,\psi}(\gamma,B,r,\theta)
\Bigr].
\end{equation*}
Using Eq. \ref{eq:model-joint} and Eq. \ref{eq:guide-factorization} we expand
\begin{align}
\mathrm{ELBO}(\phi,\psi)
&=
\mathbb{E}_{q_\psi(r)\,q_\phi(\gamma\mid r)\,p(B\mid \gamma)\,q_\phi(\theta\mid A)}
\Bigl[
\log p(\mathcal{D}\mid A,\theta)
+\log p(\theta\mid A)
+\log p(B\mid \gamma) \nonumber\\
&\qquad\qquad\qquad\qquad
+ \log p(r)
+\log p(\gamma) -\log q_\phi(\theta\mid A)
-\log p(B\mid \gamma) \nonumber\\
&\qquad\qquad\qquad\qquad
-\log q_\psi(r)
-\log q_\phi(\gamma\mid r)
\Bigr].
\label{eq:ELBO-expanded}
\end{align}
We first integrate over $\theta$ conditionally on the hard adjacency $A$.  The terms that depend on $\theta$ are $\log p(\mathcal{D}\mid A,\theta)+\log p(\theta\mid A)-\log q_\phi(\theta\mid A)$. Therefore
\begin{align}
\mathrm{ELBO}(\phi,\psi)
&=
\mathbb{E}_{q_\psi(r)\,q_\phi(\gamma\mid r)\,p(B\mid \gamma)}
\Bigl[
\mathbb{E}_{q_\phi(\theta\mid A)}
\bigl[\log p(\mathcal{D}\mid A,\theta)\bigr]
-
\KL\!\bigl(q_\phi(\theta\mid A)\,\|\,p(\theta\mid A)\bigr) \nonumber\\
&\qquad\qquad\qquad\qquad
+\ \log p(B\mid \gamma) - \log p(B\mid \gamma)
+\ \log p(r) - \log q_\psi(r)
\Bigr] \nonumber\\
&\qquad\qquad\qquad\qquad
- \mathbb{E}_{q_\psi(r)}
\bigl[\KL\!\bigl(q_\phi(\gamma\mid r)\,\|\,p(\gamma)\bigr)\bigr].
\label{eq:ELBO-grouped-pre}
\end{align}
The model and guide share the same hard Bernoulli conditional \(p(B\mid\gamma)\), so its two log-density terms cancel exactly. Separating the remaining terms for \(r\) gives
\begin{align}
\mathrm{ELBO}(\phi,\psi)
&=
\mathbb{E}_{q_\psi(r)\,q_\phi(\gamma\mid r)\,p(B\mid \gamma)}
\Bigl[
\mathbb{E}_{q_\phi(\theta\mid A)}
\bigl[\log p(\mathcal{D}\mid A,\theta)\bigr]
-
\KL\!\bigl(q_\phi(\theta\mid A)\,\|\,p(\theta\mid A)\bigr)
\Bigr] \nonumber\\
&\qquad
-\ \KL\!\bigl(q_\psi(r)\,\|\,p(r)\bigr)
\ -\ \mathbb{E}_{q_\psi(r)}\bigl[\KL\!\bigl(q_\phi(\gamma\mid r)\,\|\,p(\gamma)\bigr)\bigr].
\label{eq:ELBO-grouped}
\end{align}
The domain informed prior is now incorporated in $p(\gamma)$ (see Derivation \ref{der:logistic-beta-prior}) as
\begin{equation*}
\label{eq:logistic-beta-prior-derivation}
p(\gamma)
=
\prod_{i\neq j}
\frac{1}{\mathcal{B}(\alpha_{ij},\beta_{ij})}\,
\sigmoid(\gamma_{ij})^{\alpha_{ij}}\,
\bigl(1-\sigmoid(\gamma_{ij})\bigr)^{\beta_{ij}},
\qquad
\sigmoid(x)=\frac{1}{1+e^{-x}},
\end{equation*}
Using Eq. \ref{eq:q-gamma-flow}, we reparameterize the expectation over $q_\phi(\gamma\mid r)$ by $z\sim p_0(z)$ and $\gamma=T_\phi(z; r)$. The conditional KL divergence for the edge logits becomes
\begin{align*}
\label{eq:KL-gamma-flow}
\KL\!\bigl(q_\phi(\gamma\mid r)\,\|\,p(\gamma)\bigr)
&=
\mathbb{E}_{z\sim p_0}\Bigl[
\log q_\phi\bigl(T_\phi(z; r)\mid r\bigr)-\log p\bigl(T_\phi(z; r)\bigr)
\Bigr]\\
&=
\mathbb{E}_{z\sim p_0}\Bigl[
\log p_0(z)-\log\!\Bigl|\det \frac{\partial T_\phi(z; r)}{\partial z}\Bigr|
-\log p\bigl(T_\phi(z; r)\bigr)
\Bigr].
\end{align*}
Substituting above equation into Eq. \ref{eq:ELBO-grouped} and changing variables in the inner expectations we get the reparameterized bound
\begin{equation}
\label{eq:ELBO-flow-hard}
\scalebox{0.85}{$
\begin{aligned}
\mathrm{ELBO}(\phi,\psi)
&=
\mathbb{E}_{r\sim q_\psi(r)}\;
\mathbb{E}_{z\sim p_0}\;
\mathbb{E}_{B\sim p(\cdot\mid T_\phi(z; r))}
\Bigl[
\underbrace{
\mathbb{E}_{q_\phi(\theta\mid A)}
\bigl[\log p(\mathcal{D}\mid A,\theta)\bigr]
}_{\text{expected log likelihood}}
-
\underbrace{
\KL\!\bigl(q_\phi(\theta\mid A)\,\|\,p(\theta\mid A)\bigr)
}_{\text{KL divergence for }\theta} \\
&\quad
+\ 
\log\!\Bigl|\det \frac{\partial T_\phi(z; r)}{\partial z}\Bigr|
-\log p_0(z)
+\log p\bigl(T_\phi(z; r)\bigr)
\Bigr]
\ -\ \underbrace{\KL\!\bigl(q_\psi(r)\,\|\,p(r)\bigr)}_{\text{KL divergence for potentials}}.
\end{aligned}
$}
\end{equation}

This is our revised ELBO. 

\subsection{Logistic--Beta Prior on Edge Logits induced by the Domain Informed Prior}
\label{der:logistic-beta-prior}

For each ordered pair $(i,j)$ with $i\neq j$, let the edge probability $\pi_{ij}\in(0,1)$ carry a domain informed Beta prior
\begin{equation*}
\label{eq:beta-edge-prior}
\pi_{ij}\sim \mathrm{Beta}(\alpha_{ij},\beta_{ij}),
\qquad
p(\pi_{ij})=\frac{1}{\mathcal{B}(\alpha_{ij},\beta_{ij})}\,\pi_{ij}^{\alpha_{ij}-1}(1-\pi_{ij})^{\beta_{ij}-1},
\end{equation*}
with hyperparameters constructed from our beliefs $p_{ij}\in(0,1)$ and concentration $\nu_{ij}>0$ via
\begin{equation*}
\label{eq:domain-hyperparams}
\alpha_{ij}=\nu_{ij} p_{ij}+1,
\qquad
\beta_{ij}=\nu_{ij} (1-p_{ij})+1.
\end{equation*}
In particular $\alpha_{ij}>0$ and $\beta_{ij}>0$ (indeed $\alpha_{ij},\beta_{ij}\ge 1$), so the Beta density is proper.
\noindent Let's define the edge logit
\begin{equation*}
\label{eq:logit-map}
\gamma_{ij}=\operatorname{logit}(\pi_{ij})=\log\frac{\pi_{ij}}{1-\pi_{ij}}
\quad\Longleftrightarrow\quad
\pi_{ij}=\sigmoid(\gamma_{ij})=\frac{1}{1+\mathrm{e}^{-\gamma_{ij}}}.
\end{equation*}
The map $\gamma_{ij}\mapsto\sigmoid(\gamma_{ij})$ is a $C^1$ diffeomorphism from $\mathbb{R}$ onto $(0,1)$, with strictly positive derivative, so the change of variables formula for densities applies. The Jacobian factor is
\begin{equation}
\label{eq:logit-jac}
\left|\frac{d\pi_{ij}}{d\gamma_{ij}}\right|
=\sigmoid(\gamma_{ij})\bigl(1-\sigmoid(\gamma_{ij})\bigr),
\end{equation}
which is positive, so the absolute value may be dropped.
By the change of variables formula, the prior density of edge logit $\gamma_{ij}$ induced by the Beta prior on $\pi_{ij}$ is
\begin{align*}
\label{eq:logistic-beta-single}
p(\gamma_{ij})
&= p(\pi_{ij})\bigg|_{\pi_{ij}=\sigmoid(\gamma_{ij})}
\left|\frac{d\pi_{ij}}{d\gamma_{ij}}\right| \nonumber\\
&=\frac{1}{\mathcal{B}(\alpha_{ij},\beta_{ij})}\;
\underbrace{\sigmoid(\gamma_{ij})^{\alpha_{ij}-1}}_{\text{from }p(\pi_{ij})}\;
\underbrace{\bigl(1-\sigmoid(\gamma_{ij})\bigr)^{\beta_{ij}-1}}_{\text{from }p(\pi_{ij})}\;
\underbrace{\sigmoid(\gamma_{ij})\bigl(1-\sigmoid(\gamma_{ij})\bigr)}_{\text{Jacobian in \eqref{eq:logit-jac}}}
\nonumber\\
&=\frac{1}{\mathcal{B}(\alpha_{ij},\beta_{ij})}\;
\sigmoid(\gamma_{ij})^{\alpha_{ij}}\,
\bigl(1-\sigmoid(\gamma_{ij})\bigr)^{\beta_{ij}}.
\end{align*}
This is the Logistic-Beta density on edge logits. It is a proper density since $\gamma_{ij}\mapsto\sigmoid(\gamma_{ij})$ is a $C^1$ diffeomorphism of $\mathbb{R}$ onto $(0,1)$ with Jacobian \eqref{eq:logit-jac}, the substitution $\pi_{ij}=\sigmoid(\gamma_{ij})$ gives
\begin{equation*}
\int_{\mathbb{R}}p(\gamma_{ij})\,d\gamma_{ij}
=\int_0^1 p(\pi_{ij})\,d\pi_{ij}=1.
\end{equation*}
Assuming independence across ordered pairs in the prior,
\begin{equation*}
\label{eq:logistic-beta-product}
p(\gamma)
=\prod_{i\neq j} p(\gamma_{ij})
=\prod_{i\neq j}
\frac{1}{\mathcal{B}(\alpha_{ij},\beta_{ij})}\;
\sigmoid(\gamma_{ij})^{\alpha_{ij}}\bigl(1-\sigmoid(\gamma_{ij})\bigr)^{\beta_{ij}}.
\end{equation*}

\subsection{Closed Form KL for Gaussian $p(\theta\mid A)$}
\label{der:KL-gaussian-diag}
Let $d_\theta$ be the dimension of $\theta$. Suppose
\[
q_\phi(\theta\mid A) \;=\; \mathcal{N}\!\bigl(\mu_q(A),\,\Sigma_q(A)\bigr),
\qquad
p(\theta\mid A) \;=\; \mathcal{N}\!\bigl(\mu_p(A),\,\Sigma_p(A)\bigr),
\]
where $\Sigma_q(A)=\mathrm{diag}\!\bigl(\sigma_\phi^2(A)\bigr)$ and $\Sigma_p(A)=\mathrm{diag}\!\bigl(\sigma_p^2(A)\bigr)$ are diagonal with strictly positive entries $\sigma_{\phi,k}(A)>0$ and $\sigma_{p,k}(A)>0$ for all $k$ (so that $\Sigma_q,\Sigma_p$ are positive definite and invertible), and $\mu_q(A)=\mu_\phi(A)$. These conditions hold in our model, since $\sigma_\phi=\mathrm{softplus}(\rho_\phi)>0$ and $\sigma_p$ is a fixed positive prior scale. We derive a closed form for
\[
\KL\!\bigl(q_\phi(\theta\mid A)\,\|\,p(\theta\mid A)\bigr)
\;=\;
\E_{q_\phi(\theta\mid A)}\bigl[\log q_\phi(\theta\mid A)-\log p(\theta\mid A)\bigr].
\]

\noindent For any mean $\mu$ and positive definite covariance $\Sigma$,
\[
\log \mathcal{N}(\theta;\mu,\Sigma)
= -\frac{1}{2}\Bigl(
d_\theta\log(2\pi) + \log|\Sigma| + (\theta-\mu)^\top \Sigma^{-1} (\theta-\mu)
\Bigr).
\]
Therefore,
\begin{align*}
\log q_\phi(\theta\mid A)
&= -\frac{1}{2}\Bigl(
d_\theta\log(2\pi) + \log|\Sigma_q| + (\theta-\mu_q)^\top \Sigma_q^{-1} (\theta-\mu_q)
\Bigr),\\
\log p(\theta\mid A)
&= -\frac{1}{2}\Bigl(
d_\theta\log(2\pi) + \log|\Sigma_p| + (\theta-\mu_p)^\top \Sigma_p^{-1} (\theta-\mu_p)
\Bigr),
\end{align*}
where we suppress the explicit dependence on $A$ inside $\mu_q,\mu_p,\Sigma_q,\Sigma_p$ for readability.
Now taking the expectation under $q_\phi(\theta\mid A)$, the $d_\theta\log(2\pi)$ terms cancel and we have
\begin{align}
\KL(q\|p)
&=
\E_q[\log q - \log p] \nonumber\\
&= \frac{1}{2}\Bigl(
\log|\Sigma_p| - \log|\Sigma_q|
\Bigr)
+\frac{1}{2}\E_q\!\bigl[(\theta-\mu_p)^\top \Sigma_p^{-1} (\theta-\mu_p)\bigr]\nonumber \\
&-\frac{1}{2}\E_q\!\bigl[(\theta-\mu_q)^\top \Sigma_q^{-1} (\theta-\mu_q)\bigr].
\label{eq:KL-skeleton}
\end{align}

\noindent Let $\delta = \mu_q - \mu_p$. Since $(\theta-\mu_p)=(\theta-\mu_q)+\delta$ and $\Sigma_p^{-1}$ is symmetric, we have
\begin{align*}
(\theta-\mu_p)^\top \Sigma_p^{-1} (\theta-\mu_p)
&= (\theta-\mu_q)^\top \Sigma_p^{-1} (\theta-\mu_q)
+ 2\,\delta^\top \Sigma_p^{-1} (\theta-\mu_q)
+ \delta^\top \Sigma_p^{-1} \delta.
\end{align*}
Taking $\E_q$ and using the Gaussian moments $\E_q[\theta-\mu_q]=0$ and $\E_q\bigl[(\theta-\mu_q)(\theta-\mu_q)^\top\bigr]=\Sigma_q$ (so that $\E_q[(\theta-\mu_q)^\top M (\theta-\mu_q)] = \mathrm{tr}(M\Sigma_q)$ for any fixed $M$) gives the standard identity:
\begin{equation}
\label{eq:quad-expansion}
\E_q\!\bigl[(\theta-\mu_p)^\top \Sigma_p^{-1} (\theta-\mu_p)\bigr]
= \mathrm{tr}\!\bigl(\Sigma_p^{-1}\Sigma_q\bigr) + \delta^\top \Sigma_p^{-1}\delta.
\end{equation}
Similarly,
\begin{equation}
\label{eq:quad-self}
\E_q\!\bigl[(\theta-\mu_q)^\top \Sigma_q^{-1} (\theta-\mu_q)\bigr]
= \mathrm{tr}\!\bigl(\Sigma_q^{-1}\Sigma_q\bigr)
= d_\theta.
\end{equation}

\noindent Now plugging Eq. \ref{eq:quad-expansion} and Eq. \ref{eq:quad-self} into Eq. \ref{eq:KL-skeleton} gives
\begin{equation}
\label{eq:KL-general-appendix}
\KL(q\|p)
= \frac{1}{2}\Bigl(
\log|\Sigma_p| - \log|\Sigma_q|
- d_\theta
+ \mathrm{tr}\!\bigl(\Sigma_p^{-1}\Sigma_q\bigr)
+ (\mu_q-\mu_p)^\top \Sigma_p^{-1} (\mu_q-\mu_p)
\Bigr).
\end{equation}
Since $\Sigma_q=\mathrm{diag}(\sigma_\phi^2)$ and $\Sigma_p=\mathrm{diag}(\sigma_p^2)$ are diagonal with strictly positive entries, we have $\Sigma_p^{-1}=\mathrm{diag}(1/\sigma_p^2)$ and
\[
\log|\Sigma_p| - \log|\Sigma_q|
= \sum_{k=1}^{d_\theta}\log \sigma_{p,k}^2 - \sum_{k=1}^{d_\theta}\log \sigma_{\phi,k}^2
= 2\sum_{k=1}^{d_\theta}\log\frac{\sigma_{p,k}}{\sigma_{\phi,k}},
\]
where the last equality uses $\log\sigma^2 = 2\log\sigma$ (valid since $\sigma_{p,k},\sigma_{\phi,k}>0$), together with
\[
\mathrm{tr}\!\bigl(\Sigma_p^{-1}\Sigma_q\bigr)
= \sum_{k=1}^{d_\theta}\frac{\sigma_{\phi,k}^2}{\sigma_{p,k}^2},
\qquad
(\mu_q-\mu_p)^\top \Sigma_p^{-1} (\mu_q-\mu_p)
= \sum_{k=1}^{d_\theta}\frac{(\mu_{\phi,k}-\mu_{p,k})^2}{\sigma_{p,k}^2}.
\]

\noindent Substituting the diagonal expressions into Eq. \ref{eq:KL-general-appendix} and restoring the explicit $A$ dependence we get,
\begin{equation*}
\KL\!\bigl(q_\phi(\theta\mid A)\,\|\,p(\theta\mid A)\bigr)
=\frac{1}{2}\sum_{k=1}^{d_\theta}
\left[
\frac{\sigma_{\phi,k}^2(A)}{\sigma_{p,k}^2(A)}
+\frac{\bigl(\mu_{\phi,k}(A)-\mu_{p,k}(A)\bigr)^2}{\sigma_{p,k}^2(A)}
-1 + 2\log\frac{\sigma_{p,k}(A)}{\sigma_{\phi,k}(A)}
\right].
\end{equation*}
This is the desired closed form.

\subsection{Pseudo Code and Computational Complexity}
\label{app:pseudo_code}

\begin{algorithm}[H]
\caption{\texttt{SVI-DAG} Algorithm}
\label{alg:svidag}
\begin{algorithmic}[1]
\Require Dataset \(\mathcal{D}\); priors \((\alpha_{ij},\beta_{ij})\); flow \(T_\phi\); hypernetwork \(H_\phi\); base \(p_0(z)\); iterations \(T_{\mathrm{itr}}\); Sinkhorn iterations \(K_S\); objective draws \(S_{\mathrm{MC}}\); temperatures \(T>0\), \(\tau>0\); particles \(K_\text{part}\); step sizes \(\eta_{\phi}, \eta_{r}\)
\Ensure Variational parameters \(\phi\) and posterior particles \(\{r^{(k)}\}_{k=1}^{K_\text{part}}\)
\State \textbf{Initialize} parameters \(\phi\) and particles \(\{r^{(k)}\}_{k=1}^{K_\text{part}} \sim  N(0,\sigma_r^2 I_m)\) 
\For{\(t=1,\dots,T_{\mathrm{itr}}\)}
    \State \textbf{Compute particle wise objectives and gradients}
    \For{\(k=1,\dots,K_\text{part}\)}
        \State Compute \(P_{\tau,K_S}(r^{(k)})\gets \mathcal S_{K_S}\!\big(\exp(S_0(r^{(k)})/\tau)\big)\). 
        \State Compute \(M(r^{(k)})\gets P(r^{(k)})LP(r^{(k)})^T\).
        \State Compute \(M_{\tau,K_S}(r^{(k)})\gets P_{\tau,K_S}(r^{(k)})\,L\,P_{\tau,K_S}(r^{(k)})^{T}\).
        \For{\(s=1,\dots,S_{\mathrm{MC}}\)}
            \State Sample \(z^{(k,s)}\sim p_0\) and set \(\gamma^{(k,s)}\gets T_\phi(z^{(k,s)};r^{(k)})\).
            \State Sample \(B^{(k,s)}\gets\mathbf 1\{\gamma^{(k,s)}+R^{(k,s)}>0\}\). 
            \Comment{Eq.~\ref{eq:binary-concrete-sample}}
            \State Sample \(\tilde B^{(k,s)}\gets\sigmoid((\gamma^{(k,s)}+R^{(k,s)})/T)\).
            \Comment{Eq.~\ref{eq:binary-concrete-sample}}
            \State Set \(A^{(k,s)}\gets B^{(k,s)}\odot M(r^{(k)})\odot(\Ones-\Id)\).
            \State Set \(\tilde A^{(k,s)}\gets\tilde B^{(k,s)}\odot M_{\tau,K_S}(r^{(k)})\odot(\Ones-\Id)\).
            \State \(A_{\mathrm{ST}}^{(k,s)}\gets\tilde A^{(k,s)}+\operatorname{sg}(A^{(k,s)}-\tilde A^{(k,s)})\).  \Comment{Eq.~\ref{eq:hard-soft-st-adjacency}}
            \State \((\mu_\phi^{(k,s)},\rho_\phi^{(k,s)})\gets H_\phi(A_{\mathrm{ST}}^{(k,s)})\) and set \(\sigma_\phi^{(k,s)}\gets\operatorname{softplus}(\rho_\phi^{(k,s)})\).
            \State Sample \(\epsilon^{(k,s)}\sim\mathcal N(0,I_{d_\theta})\) and set \(\theta^{(k,s)}\gets\mu_\phi^{(k,s)}+\sigma_\phi^{(k,s)}\odot\epsilon^{(k,s)}\).
            \State Draw a minibatch \(\mathcal D_b^{(k,s)}\subset\{1,\dots,N\}\) and compute
            \[
            \widehat{\mathcal L}_{\mathrm{ELL}}^{(k,s)}
            \gets\frac{N}{|\mathcal D_b^{(k,s)}|}\sum_{i=1}^{m}\sum_{n\in\mathcal D_b^{(k,s)}}\log p_i\!\bigl(x_i^{(n)}\mid g_i(x^{(n)};\theta_i^{(k,s)})\bigr).
            \]
            \State Compute \(\mathcal L_{\mathrm{KL}}^{(k,s)}(\phi;A_{\mathrm{ST}}^{(k,s)})\) \Comment{Eq.~\ref{eq:KL_closed_form}}
            \State \(\mathcal L_{\mathrm{flow}}^{(k,s)}\gets\log\!\left|\det\frac{\partial\gamma^{(k,s)}}{\partial z^{(k,s)}}\right|-\log p_0(z^{(k,s)})+\log p(\gamma^{(k,s)})\).
            \State \(\widehat{\mathcal L}^{(k,s)}\gets\widehat{\mathcal L}_{\mathrm{ELL}}^{(k,s)}-\mathcal L_{\mathrm{KL}}^{(k,s)}+\mathcal L_{\mathrm{flow}}^{(k,s)}\).   
        \EndFor
        \State \(\widehat{\mathcal L}^{(k)}\gets S_{\mathrm{MC}}^{-1}\sum_{s=1}^{S_{\mathrm{MC}}}\widehat{\mathcal L}^{(k,s)}\).  \Comment{Eq.~\ref{eq:conditional-elbo}}
        \State Target score: \(g^{(k)} \gets \nabla_{r^{(k)}}^{\mathrm{ST}} \bigl(\widehat{\mathcal L}^{(k)} + \log p(r^{(k)})\bigr)\).
        \State \(\boldsymbol\psi^{(k)}\gets\operatorname{vec}(M_{\tau,K_S}^{(k)})\).
    \EndFor
    \State \textbf{Update particles via SVGD in relaxed acyclicity space}
    \For{\(k=1,\dots,K_\text{part}\)}
        \State \(\Phi^{(k)} \gets K_\text{part}^{-1}\sum_{j=1}^{K_\text{part}}\left[k_M(\psi^{(j)},\psi^{(k)})g^{(j)}+\nabla_{r^{(j)}}k_M(\psi^{(j)},\psi^{(k)})\right]\).
        \State $r^{(k)} \gets r^{(k)} + \eta_r \Phi^{(k)}$. \Comment{Eq.~\ref{eq:svgd_update}}
    \EndFor
    \State \textbf{Update \(\phi\) with particles held fixed}
    \State \(\widehat{\mathcal J}_{\mathrm{block}}\gets K_\text{part}^{-1}\sum_k\widehat{\mathcal L}^{(k)}\). 
    \State \(\phi\gets\phi+\eta_{\phi}\nabla_\phi^{\mathrm{ST}}\widehat{\mathcal J}_{\mathrm{block}}\).
\EndFor
\State \Return \(\phi\), \(\{r^{(k)}\}_{k=1}^{K_\text{part}}\).
\end{algorithmic}
\end{algorithm}

We include the pseudo code for \texttt{SVI-DAG} in Algorithm~\ref{alg:svidag}. We choose neural spline flows~\citep{durkan2019neural} for \(T_\phi\) as they provide expressive transformations while allowing tractable forward evaluation and log Jacobian computation. The most computationally intensive components are the formation of the relaxed acyclicity mask and particle interactions in the SVGD style update.
For each particle, applying \(K_S\) Sinkhorn normalizations to an \(m\times m\) matrix costs \(\mathcal{O}(K_Sm^2)\), giving \(\mathcal{O}(K_{\text{part}}K_Sm^2)\) work per training iteration.
This is followed by the dense multiplication $P_{\tau,K_S} L P_{\tau,K_S}^T$ that forms the generally cyclic relaxed acyclicity mask, imposing a worst case cost of $\mathcal{O}(K_{\text{part}}m^3)$. Forming the hard permutation mask can instead exploit permutation structure and does not create this cubic term. The SVGD style update requires pairwise kernel evaluations and gradients across the $K_{\text{part}}$ particles. Since the kernel uses vectorized relaxed masks \(u\) of length $m^2$, this step adds $\mathcal{O}(K_{\text{part}}^2m^2)$. The particle computations before the interaction step can be batched on a GPU.
If we are to include the \(S_{\mathrm{MC}}\) hard/soft edge and adjacency constructions, the overall algorithmic complexity is \(\mathcal{O}(K_{\text{part}}m^3+K_{\text{part}}K_Sm^2+K_{\text{part}}S_{\mathrm{MC}}m^2+K_{\text{part}}^2m^2)\), which aligns with standard cubic complexity of modern approaches to causal discovery~\citep{thompson2024prodag, lorch2021dibs, bello2022dagma, annadani2023bayesdag}. This excludes the architecture and minibatch dependent costs of the flow, hypernetwork, likelihood, and their reverse mode derivatives, so it is not a bound on total runtime.

\section{Supporting Statements and Proofs}
\label{app:supporting_statements_and_proofs}

\begin{restatable}{proposition}{beta-bernoulli-concentration-to-bernoulli-prior}
\label{prop:beta-bernoulli-concentration}
Let's fix $p\in(0,1)$ and set $\alpha_{\nu}=\nu p+1$ and $\beta_{\nu}=\nu(1-p)+1$ for $\nu>0$. Let $\pi_{\nu}\sim \mathrm{Beta}(\alpha_{\nu},\beta_{\nu})$ and conditionally on $\pi_{\nu}$ we draw $B_{\nu}\sim \mathrm{Bernoulli}(\pi_{\nu})$. Then $\pi_{\nu}\to p$ in probability as $\nu\to\infty$ and $\mathbb P(B_{\nu}=1)=\E[\pi]\to p$. Hence the Beta Bernoulli distribution converges to $\mathrm{Bernoulli}(p)$ in the large confidence limit.
\end{restatable}
\begin{proof}
We have 
$$
\E[\pi_{\nu}]=\frac{\alpha_{\nu}}{\alpha_{\nu}+\beta_{\nu}}=\frac{\nu p+1}{\nu+2}\to p,
$$ 

\noindent as $\nu\to\infty$. Also

$$
\Var(\pi_{\nu})=\frac{\alpha_{\nu}\beta_{\nu}}{(\alpha_{\nu}+\beta_{\nu})^{2}(\alpha_{\nu}+\beta_{\nu}+1)}
=\frac{\nu^{2}p(1-p)+\nu+1}{(\nu+2)^{2}(\nu+3)}
\sim\frac{p(1-p)}{\nu} \to 0.
$$
The bias-variance identity gives
\[
\E\bigl[(\pi_{\nu}-p)^{2}\bigr]
=\Var(\pi_{\nu})+\bigl(\E[\pi_{\nu}]-p\bigr)^{2}
=\Var(\pi_{\nu})+\left(\frac{1-2p}{\nu+2}\right)^{2} \to 0.
\]
By Markov's inequality, for every $\varepsilon>0$,
\[
\mathbb P\bigl(|\pi_\nu-p|>\varepsilon\bigr)
\le
\frac{\mathbb E[(\pi_\nu-p)^2]}{\varepsilon^2}
\to 0.
\]
Hence $\pi_\nu\to p$ in probability. Finally, by the law of iterated expectations,
\[
\mathbb P(B_\nu=1)
=
\mathbb E\!\left[\mathbb P(B_\nu=1\mid \pi_\nu)\right]
=
\mathbb E[\pi_\nu]
=
\frac{\nu p+1}{\nu+2}
\to p.
\]
Thus the marginal law of $B_\nu$ is
$\mathrm{Bernoulli}(\mathbb E[\pi_\nu])$ and converges to
$\mathrm{Bernoulli}(p)$.
\end{proof}

\begin{restatable}{proposition}{data-processing-inequality-KL}
\label{prop:data-processing}
Let $\mathcal{X}$ be a measurable space and let $T:\mathcal{X}\to\mathcal{Y}$ be a measurable map. Let $q$ and $p$ be probability measures on $\mathcal{X}$ with pushforwards $q_T$ and $p_T$ on $\mathcal{Y}$. Then
\begin{equation}
\label{eq:KL-data-processing}
\KL(q\,\|\,p)\;\ge\;\KL(q_T\,\|\,p_T),
\end{equation}
with equality if and only if a version of the conditional distributions satisfies $q(\cdot\mid y)=p(\cdot\mid y)$ for $q_T$ almost every $y$. Specifically, for the deterministic mapping $( B, r ) \mapsto A = B \odot M(r)$, the divergence in the $(B, r)$ space provides an upper bound for the divergence in the adjacency matrix space $A$.
\end{restatable}
\begin{proof}
Let $X \sim q$ and $Y = T(X)$. By disintegration, there exist conditional probability measures $q(\cdot \mid y)$ and $p(\cdot \mid y)$ such that
\[
q(dx) = q(dx \mid y)\, q_T(dy),
\qquad
p(dx) = p(dx \mid y)\, p_T(dy).
\]
We begin by expressing the total KL divergence on the space $\mathcal{X}$ as an integral:
\begin{equation*}
\KL(q \,\|\, p) = \int_{\mathcal{X}} \log \left( \frac{q(dx)}{p(dx)} \right) q(dx).
\end{equation*}
By substituting the disintegrated forms into the integrand, we have:
\begin{equation*}
\KL(q \,\|\, p) = \int_{\mathcal{Y}} \int_{\mathcal{X} \mid y} \log \left( \frac{q_T(dy) \, q(dx \mid y)}{p_T(dy) \, p(dx \mid y)} \right) q(dx \mid y) \, q_T(dy).
\end{equation*}
We then use the properties of the logarithm to separate the marginal and conditional components:
\begin{equation*}
\KL(q \,\|\, p) = \int_{\mathcal{Y}} \log \left( \frac{q_T(dy)}{p_T(dy)} \right) q_T(dy) + \int_{\mathcal{Y}} \left[ \int_{\mathcal{X} \mid y} \log \left( \frac{q(dx \mid y)}{p(dx \mid y)} \right) q(dx \mid y) \right] q_T(dy).
\end{equation*}
The first term on the right hand side is the KL divergence between the pushforward measures, $\KL(q_T \,\|\, p_T)$. 
The second term represents the expected conditional KL divergence, which we denote as $E_{q_T} [ \KL(q(\cdot \mid y) \,\|\, p(\cdot \mid y)) ]$. 
The Gibbs inequality states that the KL divergence between any two probability measures is non negative. 
Therefore, if we assume the measures are well defined, then the conditional term satisfies:
\begin{equation*}
\int_{\mathcal{Y}} \KL(q(\cdot \mid y) \,\|\, p(\cdot \mid y)) \, q_T(dy) \ge 0.
\end{equation*}
This non negativity implies that $\KL(q \,\|\, p) \ge \KL(q_T \,\|\, p_T)$ establishing the inequality. We can observe that for the inequality to become an equality, the conditional term must vanish.  This occurs if and only if $\KL(q(\cdot \mid y) \,\|\, p(\cdot \mid y)) = 0$ for $q_T$ almost every $y$, which is equivalent to the condition that $q(\cdot \mid y) = p(\cdot \mid y)$. In the context of our algorithm, this demonstrates that minimizing the divergence in the latent parameter space $(B, r)$ minimizes a conservative upper bound on the divergence between the actual graph distributions.
\end{proof}

\begin{restatable}{theorem}{acyclicity-by-construction}
\label{thm:acyclic_construction}
Let \(r\in\R^m\) have pairwise distinct entries and let \(B\in\{0,1\}^{m\times m}\) have zero diagonal. Let's define \(M(r)=P(r)LP(r)^T\) and \(A=B\odot M(r)\). Then the directed graph with adjacency \(A\) is acyclic. Moreover for every binary DAG there exist potentials \(r\) and a free edge matrix \(B\) such that \(A=B\odot M(r)\).
\end{restatable}

\begin{proof}
See \cite{annadani2023bayesdag}.

\end{proof}

\begin{restatable}{lemma}{sinkhornpermutationmasklimit}
\label{lem:sinkhorn-limit}
Let $m\ge 2$ and let $r=(r_1,\dots,r_m)^T\in\mathbb{R}^m$ have pairwise distinct entries. We define the score matrix $S_0(r)=r\,o^T$, where $o=(m,\dots,1)^T$. For $\tau>0$ and $K_S\ge 1$, let $P_{\tau,K_S}(r)$ denote the output obtained after $K_S$ alternating row and column normalizations applied to $\exp\!\bigl(S_0(r)/\tau\bigr)$, and let
\[
P_\tau^{\star}(r)=\lim_{K_S\to\infty}P_{\tau,K_S}(r)
\]
be its Sinkhorn limit, which exists and is the unique doubly stochastic solution of the entropic assignment problem because $\exp(S_0(r)/\tau)$ has strictly positive entries (Sinkhorn--Knopp). Let $P(r)$ be the permutation matrix that sorts indices by descending potential, and set $M_\tau^{\star}(r)=P_\tau^{\star}(r)\,L\,P_\tau^{\star}(r)^T$ and $M(r)=P(r)\,L\,P(r)^T$, where $L$ is the strictly lower triangular matrix of ones. Then
\[
P_\tau^{\star}(r)\to P(r)
\quad\text{and}\quad
M_\tau^{\star}(r)\to M(r)
\quad\text{entrywise as }\tau\to 0.
\]
Equivalently, for any iteration schedule $K_S(\tau)\to\infty$ with $\bigl\|P_{\tau,K_S(\tau)}(r)-P_\tau^{\star}(r)\bigr\|_{\max}\to 0$ as $\tau\to 0$, one has $P_{\tau,K_S(\tau)}(r)\to P(r)$ and $P_{\tau,K_S(\tau)}(r)\,L\,P_{\tau,K_S(\tau)}(r)^T\to M(r)$ entrywise. 
\end{restatable}

\begin{proof}
See \cite{annadani2023bayesdag}.
\end{proof}

\section{Evaluation Metrics}
\label{app:metrics}

We assess the quality of the learned posterior distribution using four  metrics, each capturing a distinct aspect of structured recovery. Let $q(\mathcal{G})$ represent the approximate posterior over graphs and let $\mathcal{G}^{\mathrm{GT}}$ represent the ground truth causal graph. For each metric, we draw $N_e$ posterior samples $\{\mathcal{G}^{(i)}\}_{i=1}^{N_e}$ with $\mathcal{G}^{(i)} \sim q(\mathcal{G})$ and report Monte Carlo estimates.

\paragraph{Expected SHD (E-SHD).}
The $\mathrm{SHD}(\mathcal{G}, \mathcal{G}^{\mathrm{GT}})$ counts the minimum number of edge additions, deletions, and reversals required to transform the estimated graph $\mathcal{G}$ into $\mathcal{G}^{\mathrm{GT}}$. Because our method returns a distribution rather than a point estimate, we report the expected SHD under the approximate posterior:
\begin{equation*}
    \text{E-SHD} = \mathbb{E}_{\mathcal{G}\sim q(\mathcal{G})}\bigl[\mathrm{SHD}(\mathcal{G},\, \mathcal{G}^{\mathrm{GT}})\bigr] \;\approx\; \frac{1}{N_e}\sum_{i=1}^{N_e}\mathrm{SHD}\bigl(\mathcal{G}^{(i)},\, \mathcal{G}^{\mathrm{GT}}\bigr).
\end{equation*}
Lower values indicate that posterior samples are, on average, structurally closer to the ground truth. This metric penalizes both missing and spurious edges as well as orientation errors, providing a comprehensive measure of structural accuracy.

\paragraph{Expected F1 score.}
For each posterior sample $\mathcal{G}^{(i)}$, we treat every ordered pair $(i,j)$ as a binary classification problem: an edge is either present or absent. Precision is the fraction of predicted edges that appear in $\mathcal{G}^{\mathrm{GT}}$, recall is the fraction of true edges recovered by $\mathcal{G}^{(i)}$, and the F1 score is their harmonic mean. The expected F1 score averages this quantity over the approximate posterior:
\begin{equation*}
    \text{E-F1} = \mathbb{E}_{\mathcal{G}\sim q(\mathcal{G})}\bigl[\mathrm{F1}(\mathcal{G},\, \mathcal{G}^{\mathrm{GT}})\bigr] \;\approx\; \frac{1}{N_e}\sum_{i=1}^{N_e}\mathrm{F1}\bigl(\mathcal{G}^{(i)},\, \mathcal{G}^{\mathrm{GT}}\bigr).
\end{equation*}
Higher values reflect better balance between precision and recall across posterior samples.

\paragraph{Brier score.}
The Brier score evaluates the calibration of the posterior edge probabilities by measuring the mean squared error between the marginal posterior edge probabilities and the ground truth binary edge indicators. 
Let $\hat{p}_{ij} = \mathbb{E}_{\mathcal{G} \sim q(\mathcal{G})}[A_{ij}]$ denote the marginal posterior probability of edge $(i,j)$ and let $A_{ij}^{\mathrm{GT}} \in \{0,1\}$ denote the corresponding ground truth edge indicator. The Brier score is defined as
    \begin{equation*}
        \mathrm{Brier} = \frac{1}{m(m-1)} \sum_{i \neq j} \left(\hat{p}_{ij} - A_{ij}^{\mathrm{GT}}\right)^2.
    \end{equation*}
    A lower Brier score indicates superior uncertainty quantification. A perfectly calibrated posterior that assigns probability one to true edges and zero to absent edges achieves a Brier score of zero.

\paragraph{AUROC.}
The AUROC measures the discriminative ability of the posterior edge probabilities $\{\hat{p}_{ij}\}_{i \neq j}$ in distinguishing true edges from absent ones, treating the marginal probabilities as scores for a binary classifier. It equals the probability that a randomly chosen true edge receives a higher posterior probability than a randomly chosen absent edge. An AUROC of $100\%$ indicates perfect separation, while $50\%$ corresponds to chance level discrimination. Unlike the Brier score, the AUROC is threshold free and invariant to monotone transformations of the scores, making it a complementary measure of ranking quality.

\paragraph{CPDAG evaluation.}
In settings like linear synthetic data, where the DAG is identifiable only up to its MEC, point wise comparison against a single ground truth DAG would unfairly penalize methods that correctly recover the equivalence class but select a different representative. We therefore follow standard practice and convert both the ground truth DAG and each posterior sample to their corresponding CPDAGs \citep{peters2017elements} before computing the metrics above. A CPDAG represents a MEC by directing only those edges whose orientation is shared by every member of the class and leaving the remaining edges undirected. All four metrics are then evaluated at the CPDAG level, ensuring that methods are not penalized for orientation ambiguities that are inherently unresolvable from observational data alone.

\section{Code and license}

We use the following open source repositories for comparison with the baselines:
\begin{itemize}
    \item ProDAG: \href{https://github.com/ryan-thompson/ProDAG.jl}{https://github.com/ryan-thompson/ProDAG.jl} (MIT license).
    \item BayesDAG: \href{https://github.com/microsoft/Project-BayesDAG}{https://github.com/microsoft/Project-BayesDAG} (MIT license).
    \item BCD: \href{https://github.com/ermongroup/BCD-Nets}{https://github.com/ermongroup/BCD-Nets} (No license included).
    \item DIBS: \href{https://github.com/larslorch/dibs}{https://github.com/larslorch/dibs} (MIT license).
    \item DDS/VI-DP-DAG \href{https://github.com/sharpenb/Differentiable-DAG-Sampling}{https://github.com/sharpenb/Differentiable-DAG-Sampling} (No license included).
\end{itemize}


\section{Detailed Model Specifications and Hyperparameters.}
\label{app:hyperparameters}

The common hyperparameters used for benchmarking the algorithms across all experiments are summarized in Tables \ref{tab:hyperparameters_DiBS}--\ref{tab:hyperparameters_DDS}. 

\begin{table}[htbp]
    \centering
    \caption{Architectural configurations used for \texttt{DiBS} algorithm}
    \label{tab:hyperparameters_DiBS}
    \begin{tabular}{c|c}
    \hline
        \textbf{Hyperparameter} & \textbf{Value} \\
        \hline
         No. of SVGD steps & 2000 \\
         No. of particles ($K_\text{part}$) & 20 \\
         Optimizer & RMSProp \\
         Step size & 0.005 \\
         No. of likelihood gradient MC samples & 128 \\
         No. of acyclicity gradient MC samples & 32 \\
         Graph prior & Erd\H{o}s--R\'enyi \\
         \hline
        \multicolumn{2}{c}{\textbf{SVGD} } \\
        \hline
         Type & Additive Frobenius squared exponential \\
         Latent bandwidth & 5.0 \\
         \hline
         \multicolumn{2}{c}{\textbf{Latent graph relaxation} } \\
         \hline
         Acyclicity penalty slope  & 1.0 \\
         Gumbel-softmax temperature  & 1.0 \\
         \hline
    \end{tabular}
\end{table}

\begin{table}[htbp]
    \centering
    \caption{Architectural configurations used for \texttt{BayesDAG} algorithm}
    \label{tab:hyperparameters_BayesDAG}
    \begin{tabular}{c|c}
    \hline
        \textbf{Hyperparameter} & \textbf{Value} \\
        \hline
         Number of epochs & 150 \\
         Batch size & 128 \\
         Learning rate & 0.0003 \\
         Optimizer & Adam \citep{kingma2014adam} \\
         \hline

        \multicolumn{2}{c}{\textbf{Stochastic gradient MCMC sampler} } \\
        \hline
         No. of chains & 10 \\
         Noise scale on permutation & 0.1 \\
         Noise scale on weights & 0.01 \\
         \hline
         \multicolumn{2}{c}{\textbf{Nonlinear network architecture} } \\
         \hline
         No. of hidden layers & 2 \\
         Neurons per hidden layer & 128 \\
         Activation function & ReLU \\
         Layer normalization & True \\
         Residual connections & True \\
         \hline
         \multicolumn{2}{c}{\textbf{Sinkhorn} } \\
         \hline
         Sinkhorn iterations ($K_S$) & 100 \\
         \hline
         \multicolumn{2}{c}{\textbf{Sparsity parameter tuning} } \\
         \hline
         Grid & 4 values between $\lambda_{\min}$ and $\lambda_{\max}$ \\
         $(\lambda_{\min}, \lambda_{\max})$ & $(10, 10^3)$ \\
         Grid proxy fits & 25 epochs, 64 posterior samples \\
         Selection of $\lambda$ & Closest to true sparsity level \\
         Validation set size & $\lfloor 0.1 n \rfloor$ \\
         \hline
    \end{tabular}
\end{table}

\begin{table}[htbp]
    \centering
    \caption{Architectural configurations used for \texttt{SVI-DAG} algorithm}
    \label{tab:hyperparameters_SVIDAG}
    \begin{tabular}{c|c}
    \hline
        \textbf{Hyperparameter} & \textbf{Value} \\
        \hline
         Batch size & 64 \\
         Optimizer & Adam \citep{kingma2014adam} \\
         Gradient clipping (global norm) & 1.0 \\
         Straight through mask & True \\
         No. of ELBO MC samples & 1 \\
         Temperatures $(T, \tau_{\mathrm{final}})$ & (0.3, 0.1)\\
         Sinkhorn iterations ($K_S$) & 100\\
         \hline
        \multicolumn{2}{c}{\textbf{Domain informed prior} } \\
        \hline
        $\nu_\text{min}$ & 1\\
         $\kappa$ & 500\\
         $\eta$ & 2\\
         \hline
         \multicolumn{2}{c}{\textbf{Normalizing flow} } \\
         \hline
         Type & Neural Spline Flows (coupling) \citep{durkan2019neural} \\
         No. of flows & 5\\
         No. of bins & 8\\
         Domain boundary & 5\\
         \hline
        \multicolumn{2}{c}{\textbf{SVGD} } \\
         \hline
         Type & Radial Basis Function\\
         $\sigma_r$ & 1.0\\
         Particle gradient clip ($\ell_2$ norm) & 10\\
         \hline
    \end{tabular}
\end{table}

\begin{table}[htbp]
    \centering
    \caption{Architectural configurations used for \texttt{ProDAG} algorithm}
    \label{tab:hyperparameters_ProDAG}
    \begin{tabular}{c|c}
    \hline
        \textbf{Hyperparameter} & \textbf{Value} \\
        \hline
         Optimizer & Adam \citep{kingma2014adam} \\
         Learning rate & 0.1 \\
         Max no. of epochs & 1000 \\
         Early-stopping patience & 5 \\
         No. of ELBO MC samples & 100 \\
         Acyclicity threshold & 0.1 \\
         \hline
        \multicolumn{2}{c}{\textbf{Acyclicity projection} } \\
        \hline
         Total steps & 10\\
         Initial penalty ($\mu^{(1)}$) & 1\\
         Penalty update & $\mu^{(t+1)}=\mu^{(t)}/2$\\
         Initialization ($W^{(0)}$) & 0\\
         \hline
         \multicolumn{2}{c}{\textbf{Sparsity parameter grid} } \\
         \hline
         Number of grid values & 10 \\
         $\lambda^\text{min}$ & 0\\
         $\lambda^\text{max}$ & Average $l_1$ ball with $\lambda=\infty$\\
         Selection of $\lambda$ & Held-out validation MSE ($\lfloor 0.1 n \rfloor$ samples)\\
         \hline
    \end{tabular}
\end{table}

\begin{table}[htbp]
    \centering
    \caption{Architectural configurations used for \texttt{BCD Nets} algorithm}
    \label{tab:hyperparameters_BCDNets}
    \begin{tabular}{c|c}
    \hline
        \textbf{Hyperparameter} & \textbf{Value} \\
        \hline
         No. of steps & 15000 \\
         Batch size & 256 \\
         Learning rate & 0.001 \\
         Optimizer & AdaBelief \\
         \hline
        \multicolumn{2}{c}{\textbf{Permutation network ($P$)} } \\
        \hline
         Type & MLP \\
         No. of hidden layers & 2 \\
         Hidden size & 128 \\
         Activation function & GeLU \\
         \hline
         \multicolumn{2}{c}{\textbf{Gumbel--Sinkhorn relaxation} } \\
         \hline
         Temperature & Annealed $30 \to 10 \to 1$ \\
         Sinkhorn doubly-stochastic tolerance & 0.01 \\
         Sampling mode & Hard (straight-through) \\
         \hline
         \multicolumn{2}{c}{\textbf{Variational posterior over $L$} } \\
         \hline
         $L$ parameterisation & means $|$ log-stds (lower triangle $+$ noise) \\
         Per-node noise scales & True \\
         \hline
         \multicolumn{2}{c}{\textbf{Priors} } \\
         \hline
         Lower-triangular weight prior & Horseshoe \\
         Horseshoe scale & van der Pas et al.\ ($n$-independent) \\
         Expected in-degree ($\mathrm{deg}$) & 1 \\
         Log-noise prior std ($s$) & 3.0 \\
         \hline
    \end{tabular}
\end{table}

\begin{table}[htbp]
    \centering
    \caption{Architectural configurations used for \texttt{DDS} algorithm}
    \label{tab:hyperparameters_DDS}
    \begin{tabular}{c|c}
    \hline
        \textbf{Hyperparameter} & \textbf{Value} \\
        \hline
         Max no. of epochs & 100 \\
         Early-stopping patience & 20 \\
         Batch size & 64 \\
         Learning rate (autoencoder) & 0.001 \\
         Learning rate (DAG sampler) & 0.01 \\
         Optimizer & Adam \citep{kingma2014adam} \\
         \hline
        \multicolumn{2}{c}{\textbf{Probabilistic DAG sampler} } \\
        \hline
         Order type & Top-$k$ (SoftSort) \\
         Sampling mode & Hard (straight-through) \\
         Gumbel temperature & 1.0 \\
         Noise factor & 1.0 \\
         \hline
         \multicolumn{2}{c}{\textbf{Masked autoencoder} } \\
         \hline
         Architecture & Linear \\
         Hidden layer sizes & $(16,\,16,\,16)$ \\
         \hline
         \multicolumn{2}{c}{\textbf{Loss \& edge prior} } \\
         \hline
         Loss & ELBO \\
         Sparsity regulariser (Bernoulli KL) weight & 0.1 \\
         Prior edge probability ($p$) & 0.01 \\
         \hline
    \end{tabular}
\end{table}

\section{Broader impact \& limitations}
\label{app:broader-impact-limitations}

\subsubsection{Broader impact statement}

This work is concerned with understanding cause and effect relationships from data, with potential applications across empirical sciences, economics, epidemiology, and climate science. By explicitly quantifying epistemic uncertainty over causal structures and providing a principled mechanism to incorporate domain knowledge, our approach can support more calibrated decision making, expose biases present in the data, and allow reliable answers to causal queries in settings where data are scarce or the underlying DAG is not identifiable. As such, we envision this line of work to not have any significant negative impact.

\subsubsection{Limitations}

Though our approach has several strengths, it naturally has limitations. As with most Bayesian approaches that rely on variational inference, our framework lacks theoretical guarantees for exact posterior approximation.
In particular, while the combination of normalizing flows and SVGD in the relaxed acyclicity space is designed to promote mode coverage, our method does not guarantee recovery of all DAGs within MEC. It can only aim to capture as many of them as possible, and some equivalent structures may remain underrepresented in the learned posterior. 

\end{document}